\documentclass[11pt]{article}
\pdfoutput=1
\usepackage{times}
\usepackage{hyperref}
\hypersetup{
     colorlinks   = true,
     linkcolor    = blue,
     citecolor    = green
}

\usepackage[style=alphabetic,natbib=true,backend=bibtex]{biblatex}

\PassOptionsToPackage{numbers, compress}{natbib}

\usepackage[utf8]{inputenc} 
\usepackage[T1]{fontenc}    
\usepackage{url}            
\usepackage{booktabs}       
\usepackage{microtype}      

\usepackage{amsmath,amsbsy,amssymb,amsfonts,amsthm}
\usepackage{nicefrac}
\usepackage{mathtools}
\usepackage{color}
\usepackage{xspace} 
\usepackage{times}
\usepackage{graphicx,subfigure}
\usepackage{algorithm,algorithmic} 

\usepackage{hyperref}

\usepackage{bm}
\usepackage{afterpage}

\bibliography{marginal-gibbs}
\usepackage{framed}

\newlength{\defbaselineskip}
\include{notation}

\title{Improving Gibbs Sampler Scan Quality with DoGS}

\author{
    \scshape Ioannis Mitliagkas$^\dagger$~
    Lester Mackey$^*$
    \vspace{3.5mm}
    \\
    $^\dag$Department of Computer Science, Stanford University
    \\
    $^*$Microsoft Research, New England
    \vspace{3.5mm}
    \\
    \texttt{imit@stanford.edu,}
    \texttt{lmackey@microsoft.com}
}

\newcommand{\eps}{\epsilon}

\def\balign#1\ealign{\begin{align}#1\end{align}}
\def\baligns#1\ealigns{\begin{align*}#1\end{align*}}
\def\balignat#1\ealign{\begin{alignat}#1\end{alignat}}
\def\balignats#1\ealigns{\begin{alignat*}#1\end{alignat*}}
\def\bitemize#1\eitemize{\begin{itemize}#1\end{itemize}}
\def\benumerate#1\eenumerate{\begin{enumerate}#1\end{enumerate}}

\let\originalleft\left
\let\originalright\right
\renewcommand{\left}{\mathopen{}\mathclose\bgroup\originalleft}
\renewcommand{\right}{\aftergroup\egroup\originalright}

\def\tinycitep*#1{{\tiny\citep*{#1}}}
\def\tinycitealt*#1{{\tiny\citealt*{#1}}}
\def\tinycite*#1{{\tiny\cite*{#1}}}
\def\smallcitep*#1{{\scriptsize\citep*{#1}}}
\def\smallcitealt*#1{{\scriptsize\citealt*{#1}}}
\def\smallcite*#1{{\scriptsize\cite*{#1}}}

\def\mbb#1{\mathbb{#1}}
\def\mc#1{\mathcal{#1}}

\def\tbf#1{\textbf{#1}}

\def\reals{\mathbb{R}} 

\def\<{\left\langle} 
\def\>{\right\rangle}

\def\defeq{\triangleq} 

\def\texthalf{{\textstyle\frac{1}{2}}}

\def\norm#1{\left\|{#1}\right\|} 
\def\staticnorm#1{\|{#1}\|} 

\def\indic#1{\mbb{I}\left[{#1}\right]} 

\def\E{\mbb{E}} 

\def\Esubarg#1#2{\E_{#1}\left[{#2}\right]}
\def\P{\mbb{P}} 

\renewcommand{\exp}[1]{\operatorname{exp}\left(#1\right)} 
\newcommand{\staticexp}[1]{\operatorname{exp}(#1)} 
\providecommand{\argmin}{\mathop\mathrm{arg min}}

\providecommand{\diag}{\mathop\mathrm{diag}}

\newcommand{\algref}[1]{Algorithm~{\ref{alg:#1}}}
\newcommand{\appref}[1]{Appendix~{\ref{sec:#1}}}

\newcommand{\eqnref}[1]{\eqref{eqn:#1}}
\newcommand{\figref}[1]{Figure~{\ref{fig:#1}}}
\newcommand{\lemref}[1]{Lemma~{\ref{lem:#1}}}

\newcommand{\secref}[1]{Section~{\ref{sec:#1}}}

\newcommand{\thmref}[1]{Theorem~{\ref{thm:#1}}}
\newcommand{\thmsref}[1]{Theorems~{\ref{thm:#1}}}
\newcommand{\thmssref}[1]{{\ref{thm:#1}}}

\ifdefined\ispres\else
\newtheorem{theorem}{Theorem}
\newtheorem{lemma}[theorem]{Lemma}

\newtheorem{definition}[theorem]{Definition}

\renewenvironment{proof}{\noindent\textbf{Proof}\hspace*{1em}}{\qed\\}
\newenvironment{proof-sketch}{\noindent\textbf{Proof Sketch}
  \hspace*{1em}}{\qed\bigskip\\}
\newenvironment{proof-idea}{\noindent\textbf{Proof Idea}
  \hspace*{1em}}{\qed\bigskip\\}
\newenvironment{proof-of-lemma}[1][{}]{\noindent\textbf{Proof of Lemma {#1}}
  \hspace*{1em}}{\qed\\}
\newenvironment{proof-of-theorem}[1][{}]{\noindent\textbf{Proof of Theorem {#1}}
  \hspace*{1em}}{\qed\\}
\newenvironment{proof-attempt}{\noindent\textbf{Proof Attempt}
  \hspace*{1em}}{\qed\bigskip\\}
\newenvironment{proofof}[1]{\noindent\textbf{Proof of {#1}}
  \hspace*{1em}}{\qed\bigskip\\}

\fi

\newtheorem{proposition}[theorem]{Proposition}

\newcommand{\dv}{\mathcal{V}} 

\newcommand{\dtv}[1]{\norm{#1}_{\bd,\mathrm{TV}}}
\newcommand{\xset}{\mathcal{X}}
\newcommand{\bp}{{p}}
\newcommand{\bd}{{d}}
\newcommand{\bq}{{q}}
\newcommand{\be}{{e}}
\newcommand{\bv}{{v}}
\newcommand{\methodname}{{DoGS} }

\newcommand\restr[2]{{
  \left.\kern-\nulldelimiterspace 
  #1 
  \vphantom{\big|} 
  \right|_{#2} 
  }}

\begin{document}
\maketitle

\begin{abstract} 
	The pairwise influence matrix of Dobrushin has long been used as an analytical tool to bound the rate of convergence of Gibbs sampling.
	In this work, we use Dobrushin influence as the basis of a practical tool to certify and efficiently improve the quality of a discrete Gibbs sampler.  
	Our Dobrushin-optimized Gibbs samplers (DoGS) 
	offer
	customized variable selection orders for a given sampling budget and variable subset of interest,
	 explicit bounds on total variation distance to stationarity,
	and certifiable improvements over the standard systematic and uniform random scan Gibbs samplers.
	In our experiments with joint image segmentation and object recognition, Markov chain Monte Carlo maximum likelihood estimation, and Ising model inference,
	DoGS consistently deliver higher-quality inferences with significantly smaller sampling budgets than standard Gibbs samplers.

\end{abstract}


\section{Introduction}
The Gibbs sampler of~\citet{geman1984stochastic}, also known as the \emph{Glauber dynamics} or the \emph{heat-bath algorithm}, is a leading Markov chain Monte Carlo (MCMC) method for approximating expectations unavailable in closed form. 
First detailed as a technique for restoring degraded images~\citep{geman1984stochastic}, Gibbs sampling has since found diverse applications in statistical physics~\citep{janke2008monte}, stochastic optimization and parameter estimation~\citep{Geyer91}, and Bayesian inference~\citep{lunn2000winbugs}.

The hallmark of any Gibbs sampler is conditional simulation: individual variables are successively simulated from the univariate conditionals of a multivariate target distribution.  The principal degree of freedom  is the \emph{scan}, the order in which variables are sampled \cite{he2016scan}.
While it is common to employ a \emph{systematic scan}, sweeping through each variable in turn, or a \emph{uniform random scan}, sampling each variable with equal frequency, it is known that non-uniform scans can lead to more accurate inferences both in theory and in practice~\cite{liu1995covariance,levine2006optimizing}.  
This effect is particularly pronounced when certain variables are of greater inferential interest.
Past approaches to optimizing Gibbs sampler scans were based on asymptotic quality measures approximated with the output of a Markov chain~\cite{levine2005implementing,levine2006optimizing}.

In this work, we propose a computable non-asymptotic scan quality measure for discrete target distributions based on Dobrushin's notion of variable influence~\cite{dobrushin1985constructive}.
We show that for a given subset of variables, this \emph{Dobrushin variation} (DV) bounds the marginal total variation between a target distribution and $T$ steps of Gibbs sampling with a specified scan.
More generally, Dobrushin variation bounds a weighted total variation based on user-inputted importance weights for each variable.
We couple this quality measure with an efficient procedure for optimizing scan quality by minimizing Dobrushin variation.
Our \emph{Dobrushin-optimized Gibbs samplers} (\emph{DoGS}) come equipped with a guaranteed bound on scan quality, are never worse than the standard uniform random and systematic scans, and can be tailored to a target number of sampling steps and a subset of target variables.
Moreover, Dobrushin variation can be used to evaluate and compare the quality of any user-specified set of scans prior to running any expensive simulations.

The improvements achieved by DoGS are driven by an inputted matrix, $\bar{C}$, of pairwise variable influence bounds discussed in more detail in \secref{scanquality}.  
While DoGS can be used with any discrete distribution, it was designed for targets with total influence $\staticnorm{\bar{C}} < 1$, measured in any matrix norm.  This criterion is known to hold for a variety of distributions, including Ising models with sufficiently high temperatures, hard-core lattice gas models, random graph colorings \cite{hayes2006simple}, and classes of weighted constraint satisfaction problems \cite{feng2017can}.
Moreover, as we will see in \secref{bounding-influence}, suitable variable influence bounds are readily available for pairwise and binary Markov random fields.
These user-friendly bounds give rise to total influence $\staticnorm{\bar{C}} < 1$ in all of our experiments and thereby enable improvements in both inferential speed and accuracy over standard scans.

The remainder of the paper is organized as follows.
\secref{gibbs} reviews Gibbs sampling and standard but computationally intractable measures of Gibbs sampler quality.
In \secref{scanquality}, we introduce our scan quality measure and its relationship to (weighted) total variation.
We describe our procedures for selecting high-quality Gibbs sampler scans in Section~\ref{sec:improving}.
In Section~\ref{sec:experiments}, we apply our techniques to three popular applications of the Gibbs sampler: joint image segmentation and object recognition, MCMC maximum likelihood estimation with intractable gradients, and inference in the Ising model. In each case, we observe substantial improvements in full or marginal total variation over standard scans.
Section~\ref{sec:discussion} presents our conclusions and discussion of future work.

\textbf{Notation}$~$
For any vector $\bv$ and index $i$, we let $\bv_{-i}$ represent the subvector of $\bv$ with entry $v_i$ removed.
We use $\diag(\bv)$ for a square diagonal matrix with $\bv$ on the diagonal
and $\odot$ for element-wise multiplication.
The $i$-th standard basis vector is denoted by $\be_i$, $I$ represents an identity matrix, $\bm{1}$ signifies a vector of ones, and $\norm{C}$ is the spectral norm of matrix $C$.
We  use the shorthand $[p] \defeq \{1,\dots,p\}$.
\section{Gibbs sampling and total variation}
\label{sec:gibbs}

Consider a target distribution $\pi$ on a finite $p$-dimensional state space, $\xset^p$.
Our inferential goal is to approximate expectations -- means, moments, marginals, and more complex function averages, $\Esubarg{\pi}{f(X)} = \sum_{x\in\xset^p} \pi(x) f(x)$ -- under $\pi$, but we assume that both exact computation and direct sampling from $\pi$ are prohibitive due to the large number of states, $|\xset|^p$.
Markov chain Monte Carlo (MCMC) algorithms attempt to skirt this intractability by simulating a sequence of random vectors $X^0, X^1, \dots, X^T \in \xset^p$ from tractable distributions such that expectations over $X^T$ are close to expectations under $\pi$.

\subsection{Gibbs sampling}
\algref{gibbssampling} summarizes the specific recipe employed by the Gibbs sampler~\cite{geman1984stochastic}, a leading MCMC algorithm which successively simulates single variables from their tractable conditional distributions.
\begin{algorithm}[t]
  \caption{Gibbs sampling \cite{geman1984stochastic}}
  \label{alg:gibbssampling}
  \begin{algorithmic}
    \INPUT Scan $(\bq_t)_{t=1}^T$; starting distribution $\mu$; 
     single-variable \\ \quad conditionals of target distribution, $\pi( \cdot | X_{-i})$
    
\STATE Sample from starting distribution: $X^0 \sim \mu$
    \FOR{$t$ in $1,2,\ldots, T$}
      \STATE Sample variable index to update using scan: $i_t \sim \bq_t$
      \STATE Sample $X^t_{i_t}  \sim \pi(\cdot| X^{t-1}_{-i_t})$ from its conditional
      \STATE Copy remaining variables:  $X^t_{-i_t} = X^{t-1}_{-i_t}$
    \ENDFOR
    \OUTPUT Sample sequence $(X^t)_{t=0}^T$
  \end{algorithmic}
\end{algorithm}
The principal degree of freedom in a Gibbs sampler is the \emph{scan}, the sequence of $p$-dimensional probability vectors $\bq_1, \dots, \bq_T$ determining the probability of resampling each variable on each round of Gibbs sampling.
Typically one selects between the uniform random scan, $\bq_t = (1/p, \dots, 1/p)$ for all $t$, where variable indices are selected uniformly at random on each round and the systematic scan, $\bq_t = \be_{(t \bmod p)+1}$ for each $t$, which repeatedly cycles through each variable in turn.
However, non-uniform scans are known to lead to better approximations \cite{liu1995covariance,levine2006optimizing}, motivating the need for practical procedures for evaluating and improving Gibbs sampler scans.

\subsection{Total variation}
Let $\pi_t$ represent the distribution of the $t$-th step, $X^t$, of a Gibbs sampler.
The quality of a $T$-step Gibbs sampler and its scan is typically measured in terms of total variation (TV) distance between $\pi_T$
and the target distribution $\pi$:
\begin{definition}
The \emph{total variation distance} between probability measures $\mu$ and $\nu$ is the maximum difference in expectations over all $[0,1]$-valued functions,
\begin{equation*} 
\| \mu - \nu \|_{TV} \triangleq 
\sup_{f: \xset^p \to [0,1]} | \mathbb{E}_\mu[f(X)] - \mathbb{E}_\nu[f(Y)]  | .
\end{equation*}
\end{definition}
We view TV as providing a bound on the bias of a large class of Gibbs sampler expectations; note, however, that TV does not control the variance of these expectations.

%
%

\subsection{Marginal and weighted total variation}

While we typically sample all $p$ variables in the process of Gibbs sampling, it is common for some variables to be of greater interest than others.  
For example, when modeling a large particle system, we may be interested principally in the behavior in local region of the system;
likewise, when segmenting an image into its component parts, a particular region, like the area surrounding a face, is often of primary interest.
In these cases, it is more natural to consider a marginal total variation that measures the discrepancy in expectation over only those variables of interest.  
\begin{definition}[Marginal total variation]
The \emph{marginal total variation} between probability measures $\mu$ and $\nu$ on a subset of variables $S \in [p]$ is the maximum difference in expectations over all $[0,1]$-valued functions of $\restr{X}{S}$, the restriction of $X$ to the coordinates in $S$:
\begin{equation*}
\| \mu - \nu \|_{S,TV} \triangleq 
\sup_{f :\xset^{|S|}\rightarrow [0,1]  }  
		\Big| \mathbb{E}_\mu\left[f\left(\restr{X}{S}\right)\right] 
			- 
		  \mathbb{E}_\nu\left[f\left(\restr{Y}{S}\right)\right]  \Big| .
\end{equation*}
\end{definition}
More generally, we will seek to control an arbitrary user-defined weighted total variation that assigns an independent non-negative weight to each variable and hence controls the approximation error for functions with varying sensitivities in each variable.
\begin{definition}[$\bd$-bounded differences]
We say $f:\xset^p \rightarrow \mathbb{R}$ 
has \emph{$d$-bounded differences} for $d\in\reals^d$
if, for all $X,Y \in \xset^p$,
\begin{align*}
|f(X) - f(Y)| \leq \sum_{i=1}^p d_i\indic{X_i \neq Y_i} .
\end{align*}
\end{definition}
For example, every function with range $[0,1]$ is a $\bm{1}$-Lipschitz feature, and the value of the first variable, $x \mapsto x_1$, is an $\be_1$-Lipschitz feature. This definition leads to a measure of sample quality tailored to 
$\bd$-bounded difference functions.
\begin{definition}[$\bd$-weighted total variation]
The \emph{$\bd$-weighted total variation} between probability measures $\mu$ and $\nu$ is the maximum difference in expectations across $\bd$-bounded difference functions:
\label{def:dsentitivediscrepancy}
\begin{equation*}
\| \mu - \nu \|_{\bd,\mathrm{TV}} \triangleq 
\sup_{\bd\mathrm{-bounded\ difference} f}  | \mathbb{E}_\mu[f(X)] - \mathbb{E}_\nu[f(Y)]  | 
\end{equation*}
\end{definition}


%

\section{Measuring scan quality with Dobrushin variation}
\label{sec:scanquality}
Since the direct computation of total variation measures is typically prohibitive, 
we will define an efficiently computable upper bound on the weighted total variation of Definition~\ref{def:dsentitivediscrepancy}. 
Our construction is inspired by the Gibbs sampler convergence analysis of \citet{dobrushin1985constructive}.
%

The first step in Dobrushin's approach is to control total variation in terms of coupled
random vectors, $(X_t, Y_t)_{t=0}^T$, where $X^t$ has the distribution, $\pi_t$,
of the $t$-th step of the Gibbs sampler and $Y^t$ follows the target distribution $\pi$.
For any such coupling, we can define the marginal coupling probability $\bp_{t,i} \triangleq \P(X^t_i \neq Y^t_i)$.
The following lemma, a generalization of results in \cite{dobrushin1985constructive,hayes2006simple},
shows that weighted total variation is controlled by these marginal coupling probabilities.
The proof is given in \appref{sensitivitytimeserrorproof}, and similar arguments can be found in \citet{rebeschini2014comparison}.
\begin{lemma}[Marginal coupling controls weighted TV]
\label{lem:sensitivitytimeserror}
For any joint distribution $(X,Y)$ such that $X \sim \mu$ and $Y \sim \nu$ for probability measures $\mu$ and $\nu$ on $\xset^p$
and any nonnegative weight vector $\bd \in \reals^p$, 
\begin{equation*}
\dtv{\mu - \nu}
\leq \sum_i \bd_i \P(X_i \neq Y_i).
\end{equation*}
\end{lemma}
Dobrushin's second step is to control the marginal coupling probabilities $\bp_t$ in terms of \emph{influence,} a measure of how much a change in variable $j$ affects the conditional distribution of variable $i$.
\begin{definition}[Dobrushin influence matrix]
\label{def:influence}
The \emph{Dobrushin influence} of variable $j$ on variable $i$ is given by
\begin{equation}\label{eqn:influence}
C_{ij} \triangleq \max_{(X,Y)\in N_j}
\| \pi(\cdot | X_{-i}) - \pi(\cdot | Y_{-i}) \|_{TV}
\end{equation}
where $(X,Y) \in N_j$ signifies $X_l=Y_l$\ for all $l \neq j$.
\end{definition}
This influence matrix is at the heart of our efficiently computable measure of scan quality, \emph{Dobrushin variation}.
\begin{definition}[Dobrushin variation]
\label{def:scanquality}
\label{def:dobrushin-variation}
For any nonnegative weight vector $d\in\reals^p$ and entrywise upper bound $\bar{C}$ on the Dobrushin influence \eqnref{influence}, we define the \emph{Dobrushin variation} of a scan $(\bq_t)_{t=1}^T$ as
\begin{align*}
 \dv(\bq_1,\ldots,\bq_T; \bd, \bar{C})
 &\defeq \bd^\top  B(\bq_{T})\cdots B(\bq_{1})   \bm{1} 
\end{align*}
for 
$B(\bq) \defeq (I - \mathrm{diag}(\bq)(I-\bar{C}))$.
\end{definition}
\thmref{guarantees} shows that Dobrushin variation dominates weighted TV and thereby provides target- and scan-specific guarantees on the weighted TV quality of a Gibbs sampler.
The proof in \appref{guaranteesproof} rests on the fact that, for each $t$, $b_t \defeq B(\bq_{t})\cdots B(\bq_{1}) \bm{1}$ provides an elementwise upper bound on the vector of marginal coupling probabilities, $p_t$.

\begin{theorem}[Dobrushin variation controls weighted TV]
\label{thm:guarantees}
Suppose that $\pi_T$ is the distribution of the $T$-th step of a Gibbs sampler with scan $(\bq_t)_{t=1}^T$. Then, for any nonnegative weight vector $\bd\in\reals^p$ and entrywise upper bound $\bar C$ on the Dobrushin influence \eqnref{influence},
\begin{align*}
\dtv{\pi_T - \pi}
 &\leq \dv\left((\bq_t)_{t=1}^T; \bd, \bar{C} \right).
\end{align*}

\end{theorem}

%


\section{Improving scan quality with DoGS}
\label{sec:improving}
We next present an efficient algorithm for improving the quality of any Gibbs sampler scan by minimizing Dobrushin variation.
We will refer to the resulting customized Gibbs samplers as \emph{Dobrushin-optimized Gibbs samplers} or \emph{DoGS} for short.
Algorithm~\ref{alg:coordinatedesent} optimizes Dobrushin variation using coordinate descent, with the selection distribution $\bq_t$ for each time step serving as a coordinate.
Since Dobrushin variation is linear in each $\bq_t$, each coordinate optimization (in the absence of ties) selects a degenerate distribution, a single coordinate, yielding a fully deterministic scan.
If $m \leq p$ is a bound on the size of the Markov blanket of each variable, then our forward-backward algorithm runs in time $O(\|d\|_0+\min(m\log{p}+m^2,p)T)$ with $O(p+T)$ storage for deterministic input scans.
The $T (m\log{p}+m^2)$ term arises from maintaining the derivative vector, $w$, in an efficient sorting structure, like a max-heap.

A user can initialize DoGS with any baseline scan, including a systematic or uniform random scan, and the resulting customized scan is guaranteed to have the same or better Dobrushin variation.
Moreover, DoGS scans will always be $d$-ergodic (i.e., $\dtv{\pi_T - \pi} \to 0$ as $T\to\infty$) when initialized with a systematic or uniform random scan and $\norm{\bar{C}}<1$.
This follows from the following proposition, which shows that Dobrushin variation---and hence the $d$-weighted total variation by \thmref{guarantees}---goes to $0$ under these conditions and standard scans.   
The proof relies on arguments in \cite{hayes2006simple} and is outlined in Appendix~\ref{sec:dv-decay}.
\begin{proposition}\label{prop:dv-decay}
Suppose that $\bar C$ is an entrywise upper bound on the Dobrushin influence
matrix \eqnref{influence} and that $(\bq_t)_{t=1}^T$ is a systematic or uniform random scan.
If $\norm{\bar{C}} < 1$, then, for any nonnegative weight vector $d$,
the Dobrushin variation vanishes as the chain length $T$ increases.  That is,
\baligns
\lim_{T \rightarrow \infty} \dv(\bq_1, \ldots, \bq_T; \bd, \bar{C})  = 0.
\ealigns
\end{proposition}

\begin{algorithm}[t]
  \caption{DoGS: Scan selection via coordinate descent}
  \label{alg:coordinatedesent}
  \begin{algorithmic}
	\INPUT Scan $(\bq_\tau)_{\tau=1}^T$; variable weights $\bd$; influence entrywise upper bound $\bar{C}$; (optional) target accuracy $\epsilon$.
	\STATE
	\STATE // \tbf{Forward:} Precompute coupling bounds of Section~\ref{sec:scanquality},
	\STATE // $b_t
						= B(q_t)\cdots B(q_{1})\bm{1}
						= B(q_t)b_{t-1}$ with $b_0 = \bm{1}$.
    \STATE // Only store $b = b_T$ and sequence of changes $(\Delta_t^b)_{t=0}^{T-1}$.
    \STATE // Also precompute Dobrushin variation $\mathcal{V} = d^\top b_T$
    \STATE // and derivatives $w = {\partial\mathcal{V}}{/\partial q_T} = - d \odot (I-\bar{C}) b_T$.
    \STATE $b \leftarrow \bm{1}$, $\mathcal{V} \leftarrow d^\top b$, 
    $w \leftarrow - d \odot (I-\bar{C}) b$
    \FOR{$t$ in $1,2,\ldots T$}	
    	\STATE $\Delta_{t-1}^b \leftarrow \diag{(q_t)(I-\bar{C})b}$
    	\STATE $b \leftarrow b-\Delta_{t-1}^b$
    		\hfill // $b_{t}=b_{t-1} - \Delta_{t-1}^b$
    	\STATE $\mathcal{V} \leftarrow \mathcal{V} - d^\top \Delta_{t-1}^b$
    		\hfill // $\mathcal{V}=d^\top b_t$
	\STATE  $w \leftarrow w + d \odot (I-\bar{C}) \Delta_{t-1}^b$
		\hfill // $w = - d \odot (I-\bar{C}) b_t$
	\ENDFOR
	\STATE
	\STATE // \tbf{Backward:} Optimize scan one step, $q^*_t$, at a time.
    \FOR{$t$ in $T,T-1,\ldots,1$}
      \STATE If $\mathcal{V}  \leq \epsilon$, then $q^*_t \leftarrow q_t$; \tbf{break}  \hfill // early stopping
      \STATE $b \leftarrow b + \Delta_{t-1}^b$
       \hfill // $b_{t-1} = b_t + \Delta_{t-1}^b$
      \STATE // Update $w = {\partial\mathcal{V}}{/\partial q_t} = -d_t \odot (I-\bar{C})b_{t-1}$ 
      \STATE // for $d_t^\top  \defeq d^\top B(q^*_T)\cdots B(q^*_{t+1})$ and $d_T^\top \defeq d^\top$
      \STATE $w \leftarrow w - d \odot (I-\bar{C}) \Delta_{t-1}^b$
      \STATE // Pick probability vector $q^*_t$ minimizing $d_t^\top B(q_t) b_{t-1}$ 
      \STATE $q^*_t \leftarrow e_{\argmin_i w_i}$ 
       \STATE $\mathcal{V} \leftarrow \mathcal{V} +
			d^\top\diag(q^*_t  - q_t)b
        $
        \hfill // $\mathcal{V} = d_{t-1}^\top b_{t-1}$
        \STATE ${\Delta^d}^\top \leftarrow \bd^\top \diag({q^*_t})(I - \bar{C})$
       	\STATE $\bd^\top
       					\leftarrow \bd^\top - {\Delta^d}^\top
       					$
       		\hfill // $d_{t-1}^\top=d_t^\top B(q^*_t)$
      \STATE $w \leftarrow w +  {\Delta^d} \odot (I-\bar{C}) b$
      // $w = -d_{t-1} \odot (I-\bar{C})b_{t-1}$
    \ENDFOR
    \OUTPUT Optimized scan $(q_\tau)_{\tau = 1}^{t-1}, (q^*_\tau)_{\tau=t}^T$ 
  \end{algorithmic}
\end{algorithm}

\subsection{Bounding influence}
\label{sec:bounding-influence}
An essential input to our algorithms is the entrywise upper bound $\bar C$ on the influence matrix \eqnref{influence}.
Fortunately, \citet{liu2014projecting} showed that useful influence bounds are particularly straightforward to compute for any pairwise Markov random field (MRF) target,
\begin{equation}
\label{eqn:pairwise-mrf}
\pi(X) \propto \textstyle
\staticexp{
	\sum_{i,j} \sum_{a,b \in\xset} \theta^{ij}_{ab}\, \indic{X_i = a, X_j = b}
}.
\end{equation}
\begin{theorem}[Pairwise MRF influence {\citep[Lems. 10, 11]{liu2014projecting}}]
\label{thm:pairwise-mrf}
Using the shorthand $\sigma(s) \defeq \frac{1}{1+e^{-s}}$,
the influence \eqnref{influence} of the target $\pi$ in \eqnref{pairwise-mrf} satisfies
\begin{equation*}
C_{ij} 
	\leq \max_{x_j, y_j} |2\sigma(\texthalf \max_{a,b} (\theta^{ij}_{a x_j} - \theta^{ij}_{a y_j}) - (\theta^{ij}_{b x_j} - \theta^{ij}_{b y_j}) ) - 1|.
\end{equation*}
\end{theorem}
Pairwise MRFs with binary variables $X_i \in \{-1, 1\}$ are especially common in statistical physics and computer vision.
A general parameterization for binary pairwise MRFs is given by 
\begin{equation}
\label{eqn:pairwise-model}
\pi(X) \propto \textstyle
\staticexp{
	\sum_{i\neq j} \theta_{ij} X_i X_j
	+ \sum_i \theta_i X_i
},
\end{equation}
and our next theorem, proved in \appref{influenceproofs}, leverages the strength of the singleton parameters $\theta_i$ to provide a tighter bound on the influence of these targets.
\begin{theorem}[Binary pairwise influence]
\label{thm:pairwiseinfluence}
The influence \eqnref{influence} of the target $\pi$ in \eqnref{pairwise-model} satisfies
\begin{equation*}
C_{ij} 
	\leq 	
	\frac{\left|\staticexp{2\theta_{ij}}-\staticexp{-2\theta_{ij}}\right|\,  b^*}{(1+b^*\staticexp{2\theta_{ij}})(1+b^*\staticexp{-2\theta_{ij}})}
\end{equation*}
for
$
b^* = \max (
		e^{-2\sum_{k\neq j} |\theta_{ik}|- 2 \theta_i }, 
		\min[e^{2\sum_{k\neq j} |\theta_{ik}|- 2 \theta_i },1]
).
$
\end{theorem}
\thmref{pairwiseinfluence} in fact provides an exact computation of the Dobrushin influence $C_{ij}$ whenever $b^* \ne 1$.  
The only approximation comes from the fact that the value $b^* = 1$ may not belong to the set $\mathcal{B} = \{ e^{2 \sum_{k \neq j} \theta_{ik} X_k - 2 \theta_i} \mid X \in \{-1,1\}^p \}$.  An exact computation of $C_{ij}$ would replace the cutoff of $1$ with its closest approximation in $\mathcal{B}$.

So far, we have focused on bounding influence in pairwise MRFs, as these bounds are most relevant to our experiments;
indeed, in \secref{experiments}, we will use DoGS in conjunction with the bounds of \thmsref{pairwise-mrf} and \thmssref{pairwiseinfluence} to improve scan quality for a variety of inferential tasks.
However, user-friendly bounds are also available for non-pairwise MRFs (note that any discrete distribution can be represented as an MRF with parameters in the extended reals), and we include a simple extension of \thmref{pairwiseinfluence} that applies to binary MRFs with higher-order interactions. 
Its proof is in \appref{binary-mrf-influence}
\begin{theorem}[Binary higher-order influence]
\label{thm:binary-mrf-influence}
The target 
\begin{equation*}
\pi(X) \propto \textstyle
\staticexp{
	\sum_{S\in\mc{S}} \theta_S \prod_{k \in S} X_k
	+ \sum_i \theta_i X_i
},
\end{equation*}
for $X \in\{-1,1\}^d$ and $\mc{S}$ a set of non-singleton subsets of $[p]$,
has influence \eqnref{influence} satisfying
\begin{equation*}
C_{ij} 
	\leq 	\textstyle
		\frac{|\exp{2\sum_{S\in\mc{S}: i,j \in S}|\theta_{S}|}-\exp{-2\sum_{S\in\mc{S}: i,j \in S}|\theta_{S}|}| \,b^*}{(1+b^*)^2}
\end{equation*}
for $b^* = \max(\staticexp{-2\sum_{S\in\mc{S}: i \in S, j\notin S} |\theta_{S}|  - 2 \theta_i  }, \min(\staticexp{2\sum_{S\in\mc{S}: i \in S, j\notin S} |\theta_{S}|  - 2 \theta_i  },1))$.
\end{theorem}


\subsection{Related Work}
In related work, \citet{LatuszynskiRoRo2013} recently analyzed an abstract class of adaptive Gibbs samplers parameterized by an arbitrary scan selection rule. 
However, as noted in their Rem.\ 5.13, no explicit scan selection rules were provided in that paper. 
The only prior concrete scan selection rules of which we are aware are the Minimax Adaptive Scans with asymptotic variance or convergence rate objective functions \citep{levine2006optimizing}.  Unless some substantial approximation is made, it is unclear how to implement these procedures when the target distribution of interest is not Gaussian.  

\citet{levine2006optimizing} approximate these Minimax Adaptive Scans for specific mixture models by considering single ad hoc features of interest; the approach has many hyperparameters to tune including the order of the Taylor expansion approximation, which sample points are used to approximate asymptotic quantities online, and the frequency of adaptive updating. 
Our proposed quality measure, Dobrushin variation, requires no approximation or tuning and can be viewed as a practical non-asymptotic objective function for the abstract scan selection framework of Levine and Casella.
In the spirit of \citep{LacosteHuGh2011}, DoGS can also be viewed as an approximate inference scheme calibrated for downstream inferential tasks depending only on subsets of variables.  

\citet{levine2005implementing} employ the Minimax Adaptive Scans of Levine and Casella by finding the mode of their target distribution using EM and then approximating the distribution by a Gaussian.  They report that this approach to scan selection introduces substantial computational overhead ($10$ minutes of computation for an Ising model with $64$ variables).  
As we will see in \secref{experiments}, the overhead of DoGS scan selection is manageable ($15$ seconds of computation for an Ising model with $1$ million variables) and outweighed by the increase in scan quality and sampling speed.

%
%
%
%

\section{Experiments}
\label{sec:experiments}

In this section, we demonstrate how our proposed scan quality measure and efficient optimization schemes can be used to both evaluate and improve Gibbs sampler scans when either the full distribution or a marginal distribution is of principal interest.
For all experiments with binary MRFs, we adopt the model parameterization of \eqref{eqn:pairwise-model} (with no additional temperature parameter) and use Theorem~\ref{thm:pairwiseinfluence} to produce the Dobrushin influence bound $\bar{C}$.
On all ensuing plots, the numbers in the legend state the best guarantee achieved for each algorithm plotted.
Due to space constraints, we display only one representative plot per experiment; the analogous plots from independent replicates of each experiment can be found in \appref{additional_experiments}.

\begin{figure}[htbp]
  \begin{center}
     \includegraphics[width=0.5\textwidth]{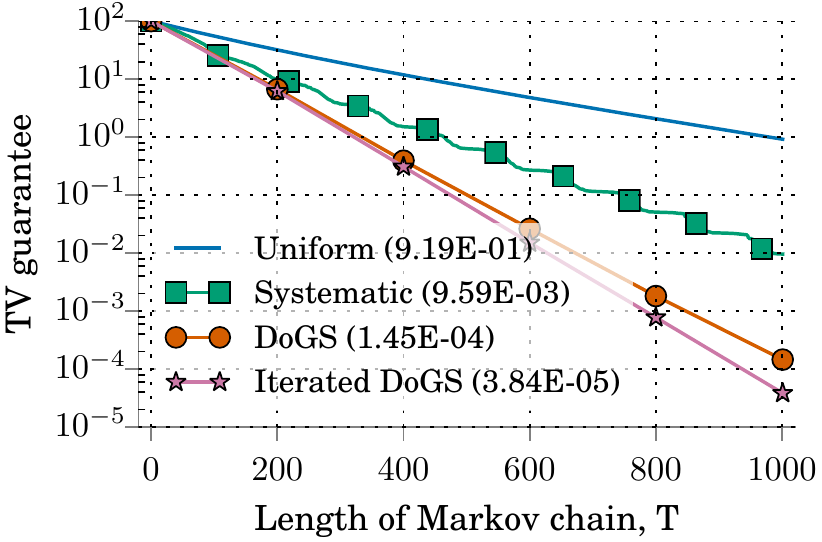}
 \end{center}
  \caption{
  TV guarantees provided Dobrushin variation for various Gibbs sampler scans on a $10\times10$ non-toroidal Ising model with random parameters (see \secref{eval}). DoGS is initialized with the systematic scan.
}
    \label{fig:isingtv}
\end{figure}
\subsection{Evaluating and optimizing Gibbs sampler scans} 
\label{sec:eval}
In our first experiment, we illustrate how Dobrushin variation can be used to select between standard scans and how DoGS can be used to efficiently improve upon standard scan quality when total variation quality is of interest.
We remind the reader that both scan evaluation and scan selection are performed offline prior to any expensive simulation from the Gibbs sampler.
Our target is a $10 \times 10$ Ising model arranged in a two-dimensional lattice, a standard model of ferromagnetism in statistical physics.  In the notation of \eqref{eqn:pairwise-model}, we draw the unary parameters $\theta_i$ uniformly at random from $\{0,1\}$, and the interaction parameters uniformly at random: $\theta_{ij} \sim \mathrm{Uniform}([0, 0.25])$. 

Figure~\ref{fig:isingtv} compares, as a function of the number of steps $T$, the total variation guarantee provided by Dobrushin variation (see \thmref{guarantees}) for the standard systematic and uniform random scans.
We see that the systematic scan, which traverses variables in row major order, obtains a significantly better TV guarantee than its uniform random counterpart for all sampling budgets $T$.  Hence, the systematic scan would be our standard scan of choice for this target.
DoGS (\algref{coordinatedesent}) initialized with $d = \bm{1}$ and the systematic scan further improves the systematic scan guarantee by two orders of magnitude.
Iterating \algref{coordinatedesent} on its own scan output until convergence (``Iterated DoGS'' in \figref{isingtv}) provides additional improvement. However, since we consistently find that the bulk of the improvement is obtained with a single run of \algref{coordinatedesent}, non-iterated DoGS remains our recommended  recipe for quickly improving scan quality.

Note that since our TV guarantee is an upper bound provided by the exact computation of Dobrushin variation, the actual gains in TV may differ from the gains in Dobrushin variation.  
In practice and as evidenced in \secref{expmarginal}, we find that the actual gains in (marginal) TV over standard scans are typically larger than the Dobrushin variation gains.

\subsection{End-to-end wall-clock time performance}
\label{sec:wall-clock}
\begin{figure}[tb]
  \begin{center}
    \includegraphics[width=0.45\textwidth]{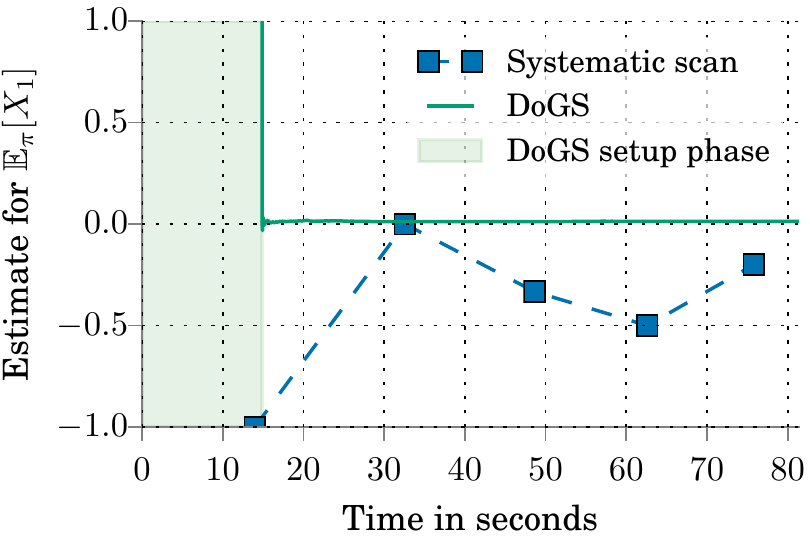}
    \includegraphics[width=0.45\textwidth]{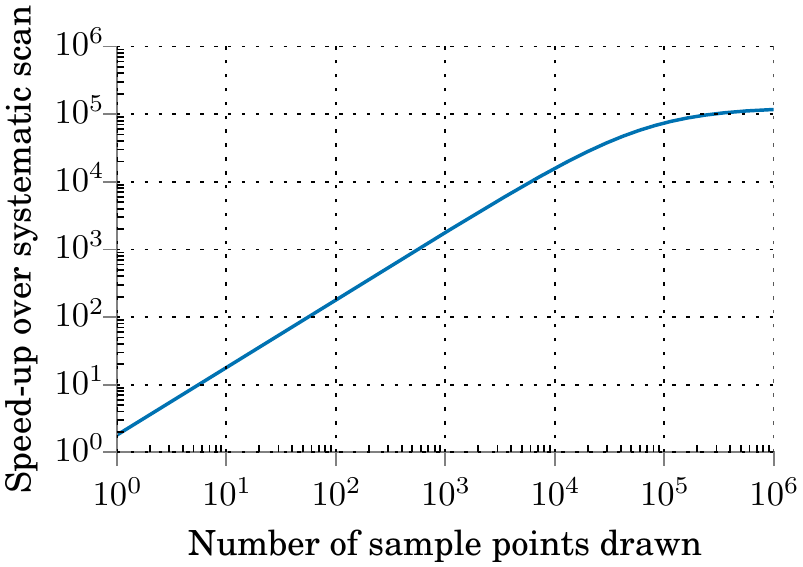}
  \end{center}
  \caption{(left) Estimate of target, $\mathbb{E}_\pi[X_1]$, versus wall-clock time for a standard row-major-order systematic scan and a DoGS optimized sequence
on an Ising model with 1 million variables (see \secref{wall-clock}).
By symmetry $\mathbb{E}_\pi[X_1]=0$. (right) The end-to-end speedup of DoGS over systematic scan, including setup and optimization time, as a function of the number of sample points we draw.}
  	  \label{fig:isingtiming}
\end{figure}
In this experiment, we demonstrate that using DoGS to optimize a scan can result in dramatic inferential speed-ups.
This effect is particularly pronounced for targets with a large number of variables and in settings that require repeated sampling from a low-bias Gibbs sampler.
The setting is the exactly same as in the previous experiment, with the exception of model size:
here we simulate a $10^3\times10^3$ Ising model, with $1$ million variables in total.
We target a single marginal $X_1$ with $d=e_1$ and take a systematic scan of length $T = 2\times 10^6$ as our input scan.
After measuring the Dobrushin variation $\eps$ of the systematic scan, 
we use an efficient length-doubling scheme to select a DoGS scan:
(0) initialize $\tilde{T}=2$;
(1) run Algorithm~\ref{alg:coordinatedesent} with the first $\tilde{T}$ steps of the systematic scan as input; (2)  if the resulting DoGS scan has Dobrushin variation less than $\eps$, we keep it; otherwise we double $\tilde{T}$ and return to step (1). 
The resulting DoGS scan has length $\tilde{T} = 16$.

We repeatedly draw independent sample points from either the length $T$ systematic scan Gibbs sampler or the length $\tilde{T}$ DoGS scan Gibbs sampler. 
Figure~\ref{fig:isingtiming} evaluates the bias of the resulting Monte Carlo estimates of $\mathbb{E}_{\pi}[X_1]$ as a function of time, including the $15$s of setup time for DoGS on this $1$ million variable model.
In comparison, \citet{levine2005implementing} report $10$ minutes of setup time for their adaptive Gibbs scans when processing a $64$ variable Ising model.
The bottom plot of \figref{isingtiming} uses the average measured time for a single step\footnote{Each Gibbs step took $12.65\mu$s on a 2015 Macbook Pro.},
the measured setup time for DoGS and the size of the two scan sequences to give an estimate of the speedup as a function of the number of sample points drawn.
Additional timing experiments are deferred to Appendix~\ref{sec:timing_appendix}.

\subsection{Accelerated MCMC maximum likelihood estimation}
\label{sec:mle}
\begin{figure}[tb]
  \begin{center}
    \includegraphics[width=0.48\textwidth]{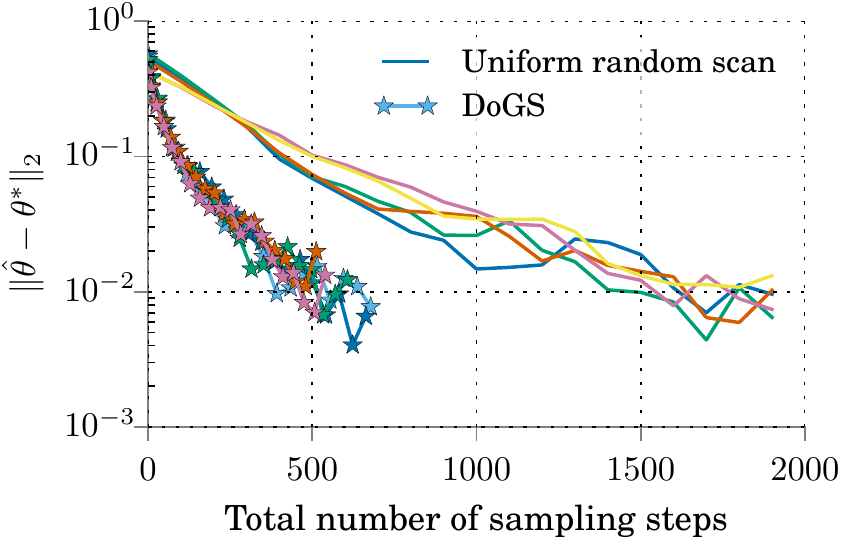}
    \includegraphics[width=0.48\textwidth]{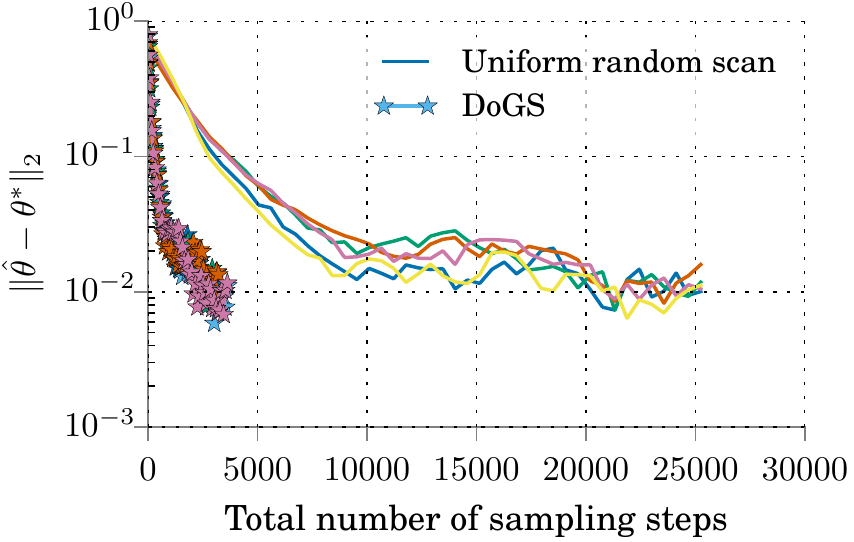}
  \end{center}
  \caption{Comparison of parameter estimation error in MCMC maximum likelihood estimation of the $3\times 3$ (left) and  a $4\times 4$ (right) Ising models of \citet{domke2015maximum}. Each MCMC gradient estimate is obtained either from the uniform random scan suggested by \Citeauthor{domke2015maximum} or from DoGS initialized with the uniform random scan, using Algorithm~\ref{alg:coordinatedesent} to achieve a target total variation of $0.01$ (see \secref{mle}). Five runs are shown in each case.
}
  	\label{fig:mll}
\end{figure}
We next illustrate how \methodname can be used to accelerate MCMC maximum likelihood estimation, while providing guarantees on parameter estimation quality.
We replicate the Ising model maximum likelihood estimation experiment of \citep[Sec.\ 6]{domke2015maximum} and show how we can provide the same level of accuracy faster.
Our aim is to learn the parameters of binary MRFs based on training samples with independent Rademacher entries.
On each step of MCMC-MLE, \Citeauthor{domke2015maximum} uses Gibbs sampling with a uniform random scan to produce an estimate of the gradient of the log likelihood.
Our DoGS variant employs \algref{coordinatedesent} with $d = \bm{1}$, early stopping parameter $\epsilon = 0.01$, and a Dobrushin influence bound constructed from the latest parameter estimate $\hat{\theta}$ using \thmref{pairwiseinfluence}.
We set the number of gradient steps, MC steps per gradient, and independent runs of Gibbs sampling to the suggested values in \cite{domke2015maximum}. After each gradient update, we record the distance between the optimal and estimated parameters.
Figure~\ref{fig:mll} displays the estimation error of five independent replicates of this experiment using each of two scans (uniform or DoGS) for two models (a $3\times 3$ and a $4 \times 4$ Ising model).
The results show that DoGS consistently achieves the desired parameter accuracy much more quickly than standard Gibbs.
%


\subsection{Customized scans for fast marginal mixing}
\label{sec:expmarginal}
In this section we demonstrate how \methodname can be used to dramatically speed up marginal inference while providing target-dependent guarantees. 
We use a $40 \times 40$ non-toroidal Ising model and set our feature to be the top left variable with $\bd = e_1$. 
Figure~\ref{fig:isingmarginalcoord} compares guarantees for a uniform random scan and a systematic scan; we also see how we can further improve the total variation guarantees by feeding a systematic scan into Algorithm~\ref{alg:coordinatedesent}.
Again we see that a single run of Algorithm~\ref{alg:coordinatedesent} yields the bulk of the improvement, and iterated applications only provide small further benefits.
For the DoGS sequence, the figure also shows a histogram of the distance of sampled variables from the target variable, $X_1$,  at the top left corner of the grid.
\begin{figure}[tb]
  \begin{center}
    \includegraphics[width=0.48\textwidth]{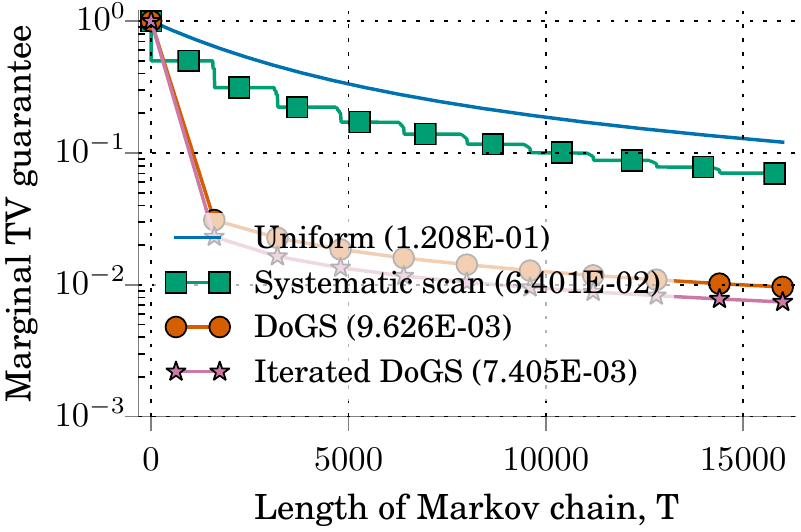}
    \includegraphics[width=0.48\textwidth]{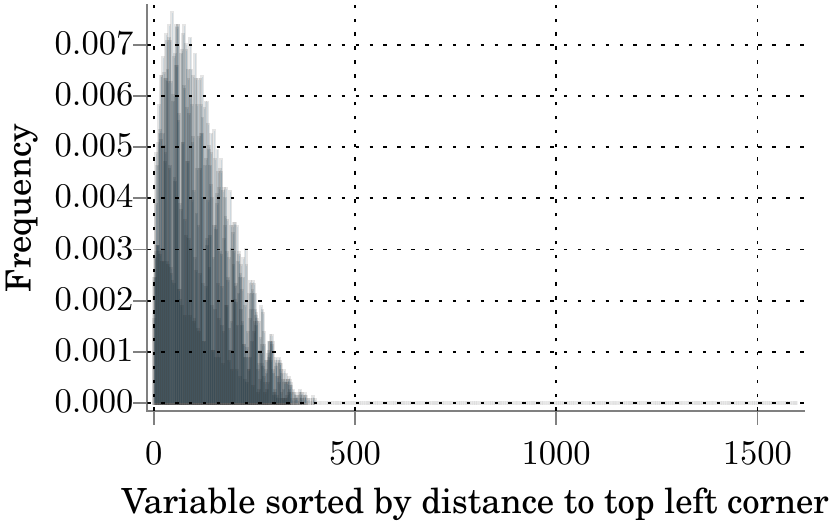}
  \end{center}
  \caption{(left) Marginal TV guarantees provided by Dobrushin variation for various Gibbs sampler scans when targeting the top left corner variable on a $40\times40$ non-toroidal Ising model with $\theta_{ij}\approx 1/3.915$ (see \secref{expmarginal}). DoGS is initialized with the systematic scan.
(right) Frequency with which each variable is sampled in the DoGS sequence of length  $16000$, sorted by Manhattan distance to target variable. 
}
\label{fig:isingmarginalcoord}
\end{figure}

Figure~\ref{fig:samplingvsguarantees} shows that optimizing our objective
actually improves performance by reducing the marginal bias much more quickly than systematic scan.
\begin{figure}[tb]
  \begin{center}
    \includegraphics[width=0.45\textwidth]{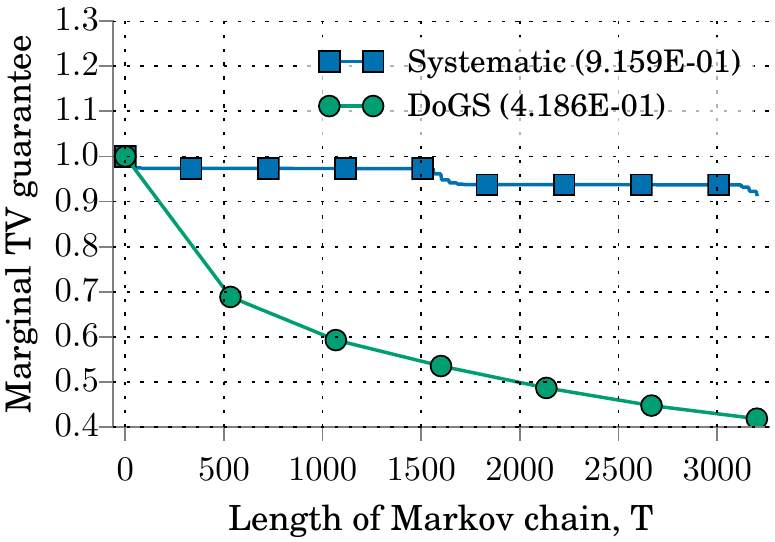}
    \includegraphics[width=0.45\textwidth]{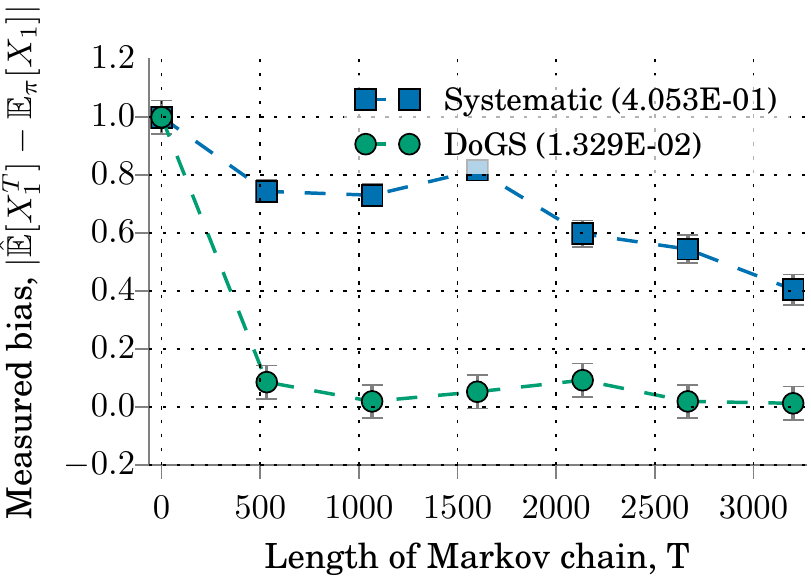}
  \end{center}
  \caption{ (left) Marginal TV guarantees provided by Dobrushin variation for systematic scan and DoGS initialized with systematic scan when targeting the top left corner variable on a $40\times40$ toroidal Ising model with $\theta_{ij}=0.25$ (see \secref{expmarginal}). 
  (right) Measured bias and standard errors from $300$ independent samples of $X_1^T$. 
  }
\label{fig:samplingvsguarantees}
\end{figure}
For completeness, we include additional experiments on a toroidal Ising model in
Appendix~\ref{sec:toroidalising}.

\subsection{Targeted image segmentation and object recognition}
\label{sec:segmentation}

The Markov field aspect model (MFAM) of \citet{verbeek2007region} is a generative model for images designed to automatically divide an image into its constituent parts (image segmentation) and label each part with its semantic object class (object recognition).
For each test image $k$, the MFAM extracts a discrete feature descriptor from each image patch $i$, assigns a latent object class label $X_i \in \xset$ to each patch, and induces the posterior distribution
\begin{align}\label{eqn:mfam-posterior}
\pi(X | y; k) 
	\propto \operatorname{exp}(
	&\textstyle \sum_{(i,j) \text{ spatial neighbors}}  \sigma \mathbb{I}\{X_i = X_j\} \\ \notag
	+ &\textstyle \sum_i \log(\sum_{a \in \mathcal{X}} \theta_{k,a} \beta_{a,y_i} \mathbb{I}\{X_i = a\}) ),
\end{align}
over the configuration of patch levels $X$.
When the Potts parameter $\sigma = 0$, this model reduces to probabilistic latent semantic analysis (PLSA) \cite{hofmann2001unsupervised}, while a positive value of $\sigma$ encourages nearby patches to belong to similar classes.
Using the Microsoft Research Cambridge (MSRC) pixel-wise labeled image database v1\footnote{http://research.microsoft.com/vision/cambridge/recognition/}, we follow the weakly supervised setup of \citet{verbeek2007region} to fit the PLSA parameters $\theta$ and $\beta$ to a training set of images and then, for each test image $k$, use Gibbs sampling to generate patch label configurations $X$ targeting the MFAM posterior \eqnref{mfam-posterior} with $\sigma = 0.48$.  
We generate a segmentation by assigning each patch the most frequent label encountered during Gibbs sampling
and evaluate the accuracy of this labeling using the Hamming error described in \cite{verbeek2007region}.
This experiment is repeated over $20$ independently generated $90\%$ training / $10\%$ test partitions of the $240$ image dataset.
%

We select our DoGS scan to target a $12\times8$ marginal patch rectangle at the center of each image (the \{0,1\} entries of $d$ indicate whether a patch is in the marginal rectangle highlighted in \figref{image_segmentation_data}) and compare its segmentation accuracy and efficiency with that of a standard systematic scan of length $T=620$.
We initialize DoGS with the systematic scan, the influence bound $\bar C$ of \thmref{pairwise-mrf}, and a target accuracy $\epsilon$ equal to  the marginal Dobrushin variation guarantee of the systematic scan.
In $11.5$ms, the doubling scheme described in Section~\ref{sec:wall-clock} produced a DoGS sequence of length $110$ achieving the  Dobrushin variation guarantee $\epsilon$ on marginal TV.
Figure~\ref{fig:image_segmentation_hamming} shows that DoGS achieves a slightly better average Hamming error than systematic scan using a $5.5\times$ shorter sequence.
Systematic scan takes $1.2$s to resample each variable of interest, while DoGS consumes $0.37$s.
Moreover, the $11.5$ms DoGS scan selection was performed only once and then used to segment all test images.
For each chain, $X^0$ was initialized to the maximum a posteriori patch labeling under the PLSA model (obtained by setting $\sigma=0$ in the MFAM).

\begin{figure}[tb]
  \begin{center}
      \includegraphics[width=0.4\textwidth]{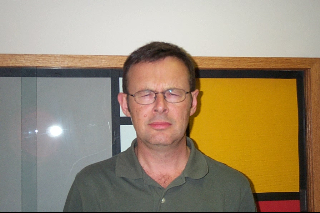}
      \includegraphics[width=0.41\textwidth]{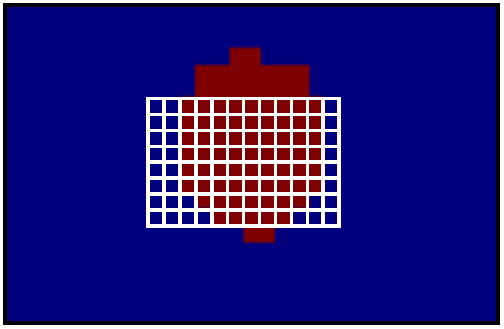}
  \end{center}
  \caption{(left) Example test image from MSRC dataset. (right) Segmentation produced by DoGS Markov field aspect model targeting the center region outlined in white (see \secref{segmentation}).}
  	  \label{fig:image_segmentation_data}
\end{figure}
\begin{figure}[tb]
  \begin{center}
    \includegraphics[width=0.6\textwidth]{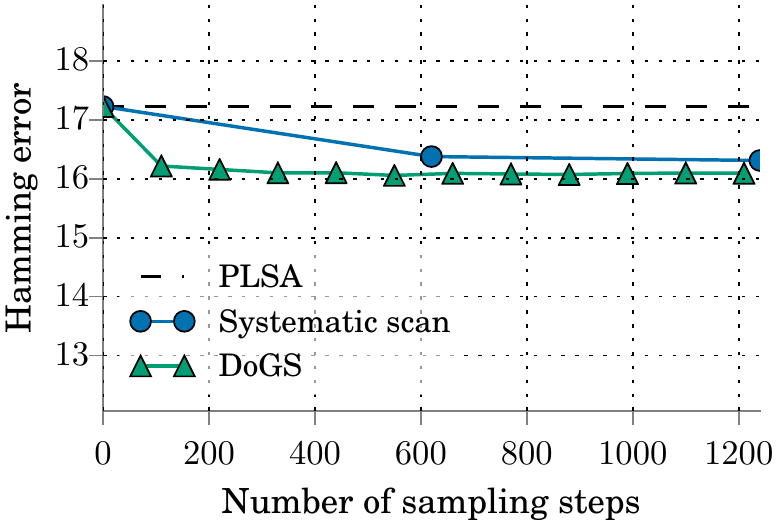}
  \end{center}
  \caption{Average test image segmentation error under the Markov field aspect model of \secref{segmentation}.
  PLSA represents the maximum a posteriori patch labeling under the MFAM \eqnref{mfam-posterior} with $\sigma = 0$.
  Errors are averaged over $20$ MSRC test sets of $24$ images.
}
  	  \label{fig:image_segmentation_hamming}
\end{figure}

\subsection{Using loose influence bounds in computations }
\label{sec:exp_loose_bound}
In our experiments so far we used Theorem~\ref{thm:pairwiseinfluence} to produce the Dobrushin influence bound $\bar{C}$.
In this section, we evaluate the performance of DoGS on marginal inference, 
when the upper bound, $\bar{C}$, used for all computations is not tight.
Figure~\ref{fig:loose_bound} shows that DoGS' performance degrades gracefully,
as the influence upper bound loosens, from left to right.
The bottom row demonstrates the quality of the empirical estimates obtained.
The results suggest that using a loose influence bound (up to $50\%$ in this case) does not lead to serious accuracy penalty.
\begin{figure}[hbp]
  \begin{center}
    \includegraphics[width=0.24\textwidth]{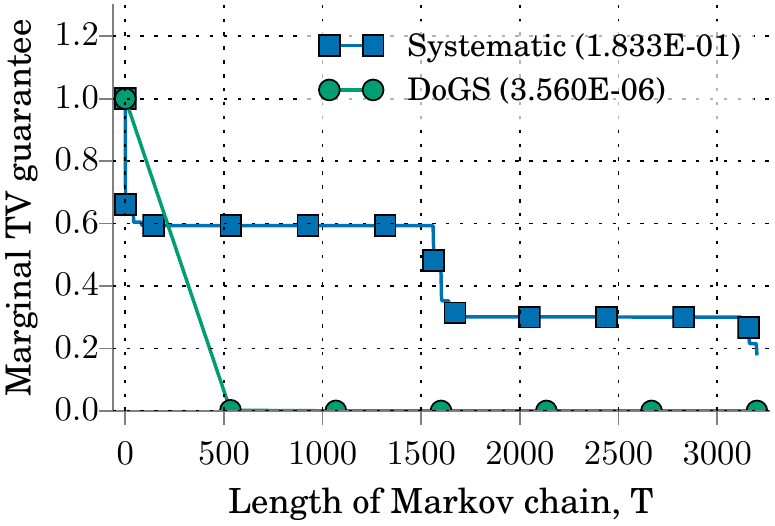}
    \includegraphics[width=0.24\textwidth]{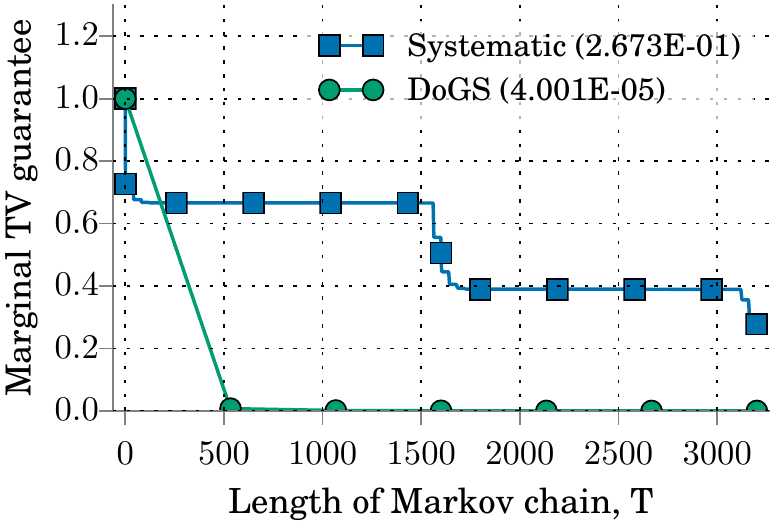}
    \includegraphics[width=0.24\textwidth]{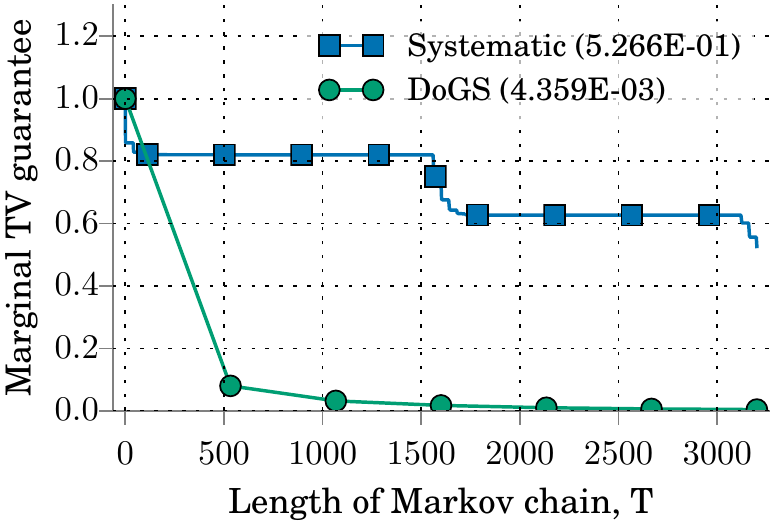}
    \includegraphics[width=0.24\textwidth]{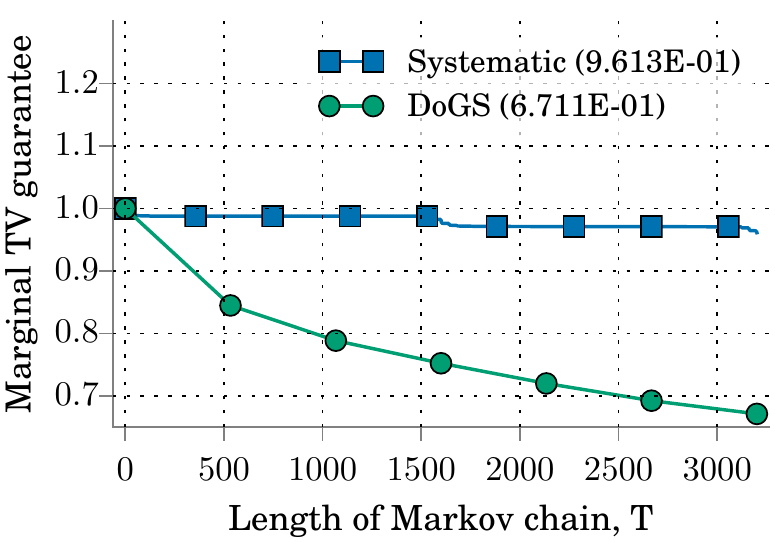}
    \includegraphics[width=0.24\textwidth]{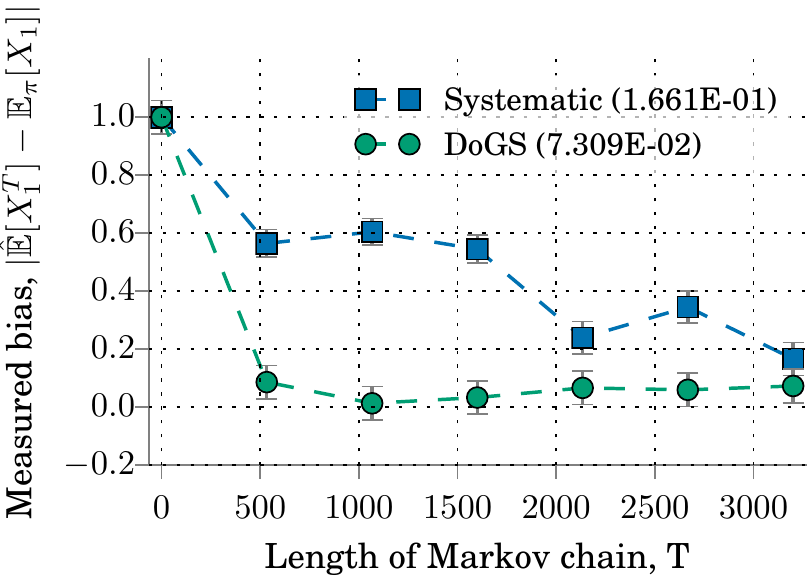}
    \includegraphics[width=0.24\textwidth]{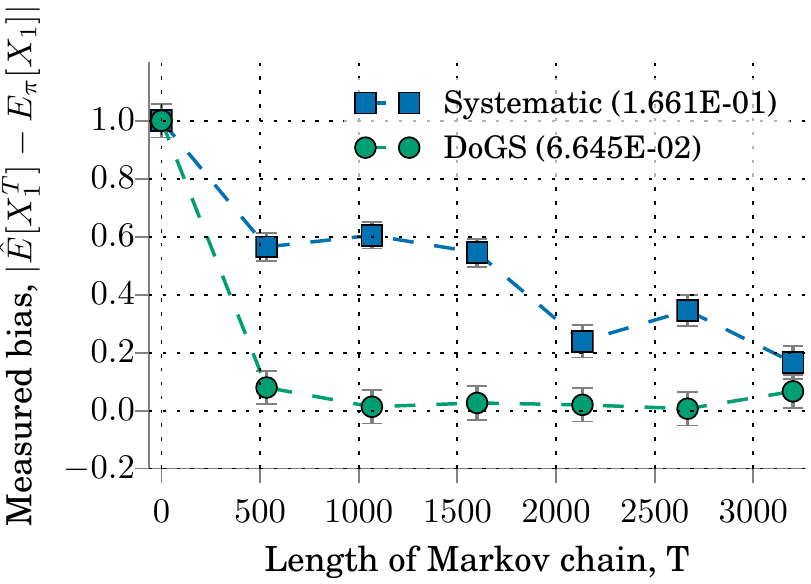}
    \includegraphics[width=0.24\textwidth]{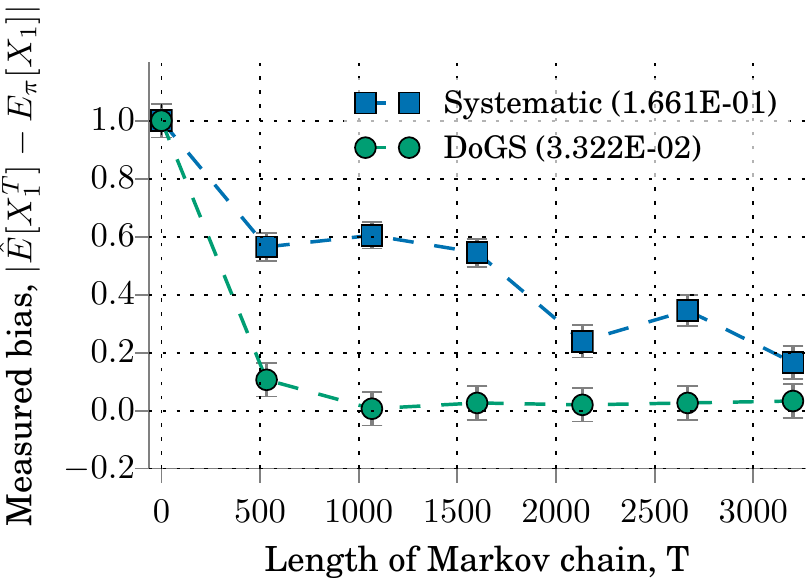}
    \includegraphics[width=0.24\textwidth]{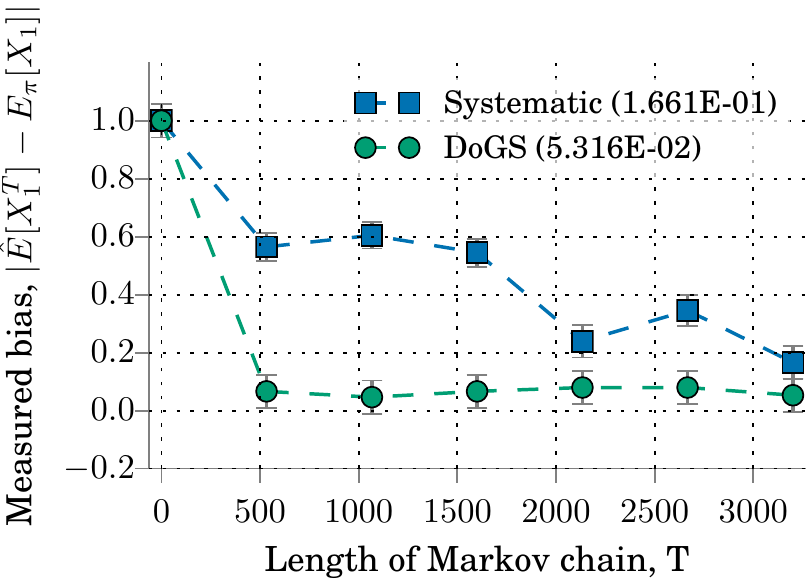}
  \end{center}
  \caption{ 
  Evaluatings DoGS on marginal inference, when the upper bound, $\bar{C}$, on  influence matrix, $C$, is not tight.
  (top row) Marginal TV guarantees provided by Dobrushin variation for systematic scan and DoGS initialized with systematic scan when targeting the top left corner variable on a $40\times40$ toroidal Ising model with $\theta_{ij} \approx 0.165$. 
  From left to right, computation uses $\bar{C}=\{1.0, 1.1, 1.3, 1.5\} \cdot \hat{C}$,
  where $\hat{C}$ denotes the bound from Theorem~\ref{thm:pairwiseinfluence}.
  (bottom row) Measured bias and standard errors from $300$ independent samples of $X_1^T$ corresponding to the setting above.
  }
  \label{fig:loose_bound}
\end{figure}


\section{Discussion}
\label{sec:discussion}
We introduced a practical quality measure -- Dobrushin variation -- for evaluating and comparing existing Gibbs sampler scans and efficient procedures -- DoGS -- for developing customized fast-mixing scans tailored to marginals or distributional features of interest. 
We deployed DoGS for three common Gibbs sampler applications -- joint image segmentation and object recognition, MCMC maximum likelihood estimation, and Ising model inference -- and in each case achieved higher quality inferences with significantly smaller sampling budgets than standard Gibbs samplers.
In the future, we aim to enlist DoGS for additional applications in computer vision and natural language processing, extend the reach of DoGS to models containing continuous variables, and integrate DoGS into large inference engines built atop Gibbs sampling.

\printbibliography

\appendix


\section{Proofs}
\label{sec:proofs}

\subsection{Proof of \lemref{sensitivitytimeserror}}
\label{sec:sensitivitytimeserrorproof}
\newcommand{\dnorm}{{\bd-\mathrm{Lip.}}}
Let $X \sim \mu$ and $Y\sim \nu$. Now define the sequence $(Z_i)_{i=0}^p$, such that $Z_0 \triangleq X$, $Z_p \triangleq Y$ and $Z_{i-1},Z_i$ can only differ on element $i$.
By Definition~\ref{def:dsentitivediscrepancy} and the compactness of the space of $d$-Lipschitz features, there exists $d$-Lipschitz $f$ such that
\begin{equation*}
\| \mu - \nu \|_{\bd,\mathrm{TV}} =
 | \mathbb{E}[f(X) - f(Y)]  | .
\end{equation*}
Then using triangle inequality,
\baligns
\left| \mathbb{E}[f(X) -f(Y)] \right|
\leq& \left|
 \mathbb{E}[\sum_{i=1}^p f(Z_{i-1}) -f(Z_{i})]
\right| \\
\leq& \sum_{i=1}^p\left|
 \mathbb{E}[ f(Z_{i-1}) -f(Z_{i})]
\right| \\
\leq&\sum_{i=1}^p \mathbb{P}( X_i \neq Y_i ) d_i \\
=& \bd^\top \bp_t
\ealigns

%


\subsection{Proof of Theorem~\ref{thm:guarantees}}
\label{sec:guaranteesproof}
First we state a useful result due to \citet{dobrushin1985constructive}.
\begin{lemma}[\citep{dobrushin1985constructive}, similar arguments can be found in Theorem~6 of \cite{hayes2006simple} and in \cite{de2016ensuring}]
\label{lem:dobrushinrecursion}
Consider the marginal coupling probability $\bp_{t,i} \triangleq \P(X^t_i \neq Y^t_i)$ and influence matrix, $C$, we defined in Section~\ref{sec:scanquality}
and an arbitrary scan sequence $(i_t)_{t=1}^{t=T}$.
Then, an application of the law of total probability yields the following bound on the marginal coupling probabilities.
\begin{equation*}
	\bp_{t,i}
	\leq 
	\left\{
	\begin{array}{ll}
	 \sum_{j \neq i} C_{ij} \bp_{t,j}, & i_t = i  \\ 
	 & \\
	\bp_{t-1,i},  & \textrm{o.w.}
	\end{array} 
	\right.
\end{equation*}
and $\bp_{0,i} \leq 1$ for all $i$.
\end{lemma}

\begin{proofof}{Theorem~\ref{thm:guarantees}}
At each time step, $i_t=i$ with probability $\bq_{t,i}$. Let $z_i(t) \triangleq \mathbb{I}\{i_t = i\}$ and 
$Z{(t)}$ denote the diagonal matrix with $Z_{ii}{(t)} = z_i(t)$ so that
$\mathbb{E}Z(t) = \mathrm{diag}(\bq_t)$.
Now, from Lemma~\ref{lem:sensitivitytimeserror}, and 
using Lemma~\ref{lem:dobrushinrecursion},
\baligns
 | \mathbb{E}f(X^T)& - \mathbb{E}f(Y^T) | \\
 \leq& \bd^\top \bp_T 
  = \sum_i \bd_i \bp_{T,i} \\
  \leq& \sum_i \bd_i \left(
  			 z_i(T) \sum_{j \neq i} C_{ij} \bp_{T-1,j} 
  			 + (1-z_i(T)) \bp_{T-1,i}
  		\right) \\
    =& \bd^\top (I - Z_T(I-C)) \bp_{T-1}.
\ealigns
Now, taking an expectation over the randomness of sampling (random variables
$i_t$), we get
\baligns
\mathbb{E} | \mathbb{E}f(X^T) - \mathbb{E}f(Y^T) | 
 \leq \bd^\top (I - \mathrm{diag}(\bq_T)(I-C)) \mathbb{E}\bp_{T-1}
 \leq \bd^\top B(\bq_T)\cdots B(\bq_1) \bm{1} 
 = \dv(\bq_1, \ldots, \bq_T; \bd, C)
 \leq \dv(\bq_1, \ldots, \bq_T; \bd, \bar{C}),
 \ealigns
where we used the fact that $\bp_0$ is a vector of probabilities, so 
all of its elements are at most $1$.

\end{proofof}

\subsection{Proof of Proposition~\ref{prop:dv-decay}}
\label{sec:dv-decay}

\begin{proof}
Let $\eps = \norm{\bar{C}}$.  From Definition~\ref{def:dobrushin-variation}, we have that
\begin{align*}
 \dv(\bq_1,\ldots,\bq_T; \bd, \bar{C})
 \defeq \bd^\top  B(\bq_{T})\cdots B(\bq_{1})   \bm{1} 
 = d^\top p_T
\end{align*}
Theorem~6 in \cite{hayes2006simple} implies that the entries of the marginal coupling probability, $p_T$ decay with rate $(1 - \epsilon/n)^T$ for uniform random scans.
Similarly, Theorem~8 of \cite{hayes2006simple} implies that the entries of the marginal coupling decay with rate $(1-\epsilon/2)^{T/n}$ for systematic scans. 
In both cases, the statement holds by taking $T$ to infinity.
\end{proof}

\subsection{Proof of Theorem~\ref{thm:pairwiseinfluence}: Influence bound for binary pairwise MRFs}
\label{sec:influenceproofs}
Our proof relies on the following technical lemma.

\begin{lemma}\label{lem:gdiff}
Consider the function $g(w,z) = 1/(1+zw)$ for $w \geq 0$ and $z \in [r,s]$ for some $s, r \geq 0$.
We have
\balign\label{eqn:gdiff}
|g(w,z) - g(w',z)| 
	= \frac{|w - w'|z}{(1+zw)(1+zw')}
	\leq \frac{|w - w'|z^*}{(1+z^*w)(1+z^*w')}
\ealign
for $z^* = \max(r,\min(s,\sqrt{1/(ww')}))$.
\end{lemma}
\begin{proof}
The inequality follows from the fact that the expression \eqnref{gdiff} is increasing in $z$ on $[0,\sqrt{1/(ww')})$ and decreasing on $(\sqrt{1/(ww')},\infty)$.
\end{proof}

\begin{proofof}{Theorem~\ref{thm:pairwiseinfluence}}
In the notation of \lemref{gdiff}, we see that, for each $i$ and $j\neq i$, the full conditional of $X_i$ is given by
\baligns
\pi(X_i=1 | X_{-i}) 
	&= \frac{1}{1 + \staticexp{-2\sum_{k} \theta_{ik} X_k - 2 \theta_i }} \\
	&= \frac{1}{1 + \staticexp{-2\sum_{k\neq j} \theta_{ik} X_k- 2 \theta_i }\staticexp{-2\theta_{ij} X_j }}  \\
	&= g(\staticexp{-2\theta_{ij} X_j }, b)
\ealigns
for $b = \staticexp{-2\sum_{k\neq j} \theta_{ik} X_k- 2 \theta_i } \in \left[\staticexp{-2\sum_{k\neq j} |\theta_{ik}|- 2 \theta_i }, \staticexp{2\sum_{k\neq j} |\theta_{ik}|- 2 \theta_i }\right]$.

Therefore, by \lemref{gdiff}, the influence of $X_j$ on $X_i$ admits the bound
\begin{align*}
C_{ij}
	\triangleq & 
		\max_{X, Y \in B_j} \Big|\pi(X_i=1 | X_{-i}) - \pi(Y_i=1 | Y_{-i})\Big| \\
	= & \max_{X, Y \in B_j} \left|g(\staticexp{-2\theta_{ij} X_j },b) - g(\staticexp{-2\theta_{ij} Y_j }, b)\right| \\
	= &
		\max_{X, Y \in B_j} 
		\frac{\left|\staticexp{-2\theta_{ij} X_j }-\staticexp{-2\theta_{ij} Y_j }\right|  b}{(1+b\staticexp{-2\theta_{ij} X_j })(1+b\staticexp{-2\theta_{ij} Y_j })}\\
	= &
		\max_{X, Y \in B_j} 
		\frac{\left|\staticexp{2\theta_{ij}}-\staticexp{-2\theta_{ij}}\right|  b}{(1+b\staticexp{2\theta_{ij}})(1+b\staticexp{-2\theta_{ij}})}\\
	\leq &
		\frac{\left|\staticexp{2\theta_{ij}}-\staticexp{-2\theta_{ij}}\right|  b^*}{(1+b^*\staticexp{2\theta_{ij}})(1+b^*\staticexp{-2\theta_{ij}})}
\end{align*}
for $b^* = \max(\staticexp{-2\sum_{k\neq j}|\theta_{ik}|- 2 \theta_i },\min(\staticexp{2\sum_{k\neq j} |\theta_{ik}|- 2 \theta_i },1))$.
\end{proofof}

\subsection{Proof of Theorem~\ref{thm:binary-mrf-influence}: Influence bound for binary higher-order MRFs}
\label{sec:binary-mrf-influence}
Mirroring the proof of \thmref{pairwiseinfluence} and adopting the notation of \lemref{gdiff}, we see that, for each $i$ and $j\neq i$, the full conditional of $X_i$ is given by
\baligns
\pi(X_i=1 | X_{-i}) 
	&= \frac{1}{1 + \staticexp{-2\sum_{S\in\mc{S}: i \in S, j\notin S} \theta_{S} \prod_{k  \in S: k\neq i} X_k - 2 \theta_i }\staticexp{-2\sum_{S\in\mc{S}: i,j \in S}\theta_{S} X_j \prod_{k  \in S: k\notin\{i, j\}} X_k }}  \\
	&\textstyle= g(\staticexp{-2\sum_{S\in\mc{S}: i,j \in S}\theta_{S} X_j \prod_{k  \in S: k\notin\{i,j\}} X_k }, b)
\ealigns
for $b = \staticexp{-2\sum_{S\in\mc{S}: i \in S, j\notin S} \theta_{S} \prod_{k  \in S: k\neq i} X_k - 2 \theta_i } \in \left[\staticexp{-2\sum_{S\in\mc{S}: i \in S, j\notin S} |\theta_{S}|  - 2 \theta_i  }, \staticexp{2\sum_{S\in\mc{S}: i \in S, j\notin S} |\theta_{S}|  - 2 \theta_i  }\right]$.

Therefore, by \lemref{gdiff}, as in the argument of \thmref{pairwiseinfluence}, the influence of $X_j$ on $X_i$ admits the bound
\begin{align*}
C_{ij}
	\triangleq & 
		\max_{X, Y \in B_j} \Big|\pi(X_i=1 | X_{-i}) - \pi(Y_i=1 | Y_{-i})\Big| \\
	= & \max_{X, Y \in B_j} \textstyle\left|g(\staticexp{-2\sum_{S\in\mc{S}: i,j \in S}\theta_{S} X_j \prod_{k  \in S: k\notin\{i,j\}} X_k }, b) - g(\staticexp{-2\sum_{S\in\mc{S}: i,j \in S}\theta_{S} Y_j \prod_{k  \in S: k\notin\{i,j\}} X_k }, b)\right| \\
	= &
		\max_{X, Y \in B_j} 
		\frac{\left|\staticexp{2\sum_{S\in\mc{S}: i,j \in S}\theta_{S}\prod_{k  \in S: k\notin\{i,j\}} X_k }-\staticexp{-2\sum_{S\in\mc{S}: i,j \in S}\theta_{S}\prod_{k  \in S: k\notin\{i,j\}} X_k }\right|  b}{(1+b\staticexp{2\sum_{S\in\mc{S}: i,j \in S}\theta_{S}\prod_{k  \in S: k\notin\{i,j\}} X_k})(1+b\staticexp{-2\sum_{S\in\mc{S}: i,j \in S}\theta_{S}\prod_{k  \in S: k\notin\{i,j\}} X_k})}\\
	= &
		\max_{X, Y \in B_j} 
		\frac{\left|\staticexp{2\sum_{S\in\mc{S}: i,j \in S}\theta_{S}\prod_{k  \in S: k\notin\{i,j\}} X_k }-\staticexp{-2\sum_{S\in\mc{S}: i,j \in S}\theta_{S}\prod_{k  \in S: k\notin\{i,j\}} X_k }\right|  b^*}{(1+b^*\staticexp{2\sum_{S\in\mc{S}: i,j \in S}\theta_{S}\prod_{k  \in S: k\notin\{i,j\}} X_k})(1+b^*\staticexp{-2\sum_{S\in\mc{S}: i,j \in S}\theta_{S}\prod_{k  \in S: k\notin\{i,j\}} X_k})}\\
	\leq &
		\frac{\left|\staticexp{2\sum_{S\in\mc{S}: i,j \in S}|\theta_{S}|}-\staticexp{-2\sum_{S\in\mc{S}: i,j \in S}|\theta_{S}|}\right|  b^*}{(1+b^*)^2}
\end{align*}
for $b^* = \max(\staticexp{-2\sum_{S\in\mc{S}: i \in S, j\notin S} |\theta_{S}|  - 2 \theta_i  }, \min(\staticexp{2\sum_{S\in\mc{S}: i \in S, j\notin S} |\theta_{S}|  - 2 \theta_i  },1))$.

\section{Additional experiments} \label{sec:additional_experiments}

We provide a few experimental results that were excluded from the main body
of the paper due to space limitations.

\subsection{Independent replicates of evaluating and optimizing Gibbs sampler scans experiment}
\label{sec:addition_ising_tv}
Figure~\ref{fig:additional_ising_tv} displays the results of nine independent replicates of the ``Evaluating and optimizing Gibbs sampler scans'' experiment of Section~\ref{sec:experiments}, with independently drawn unary and binary potentials.
\begin{figure}[hbp]
  \begin{center}
    \includegraphics[width=0.32\textwidth]{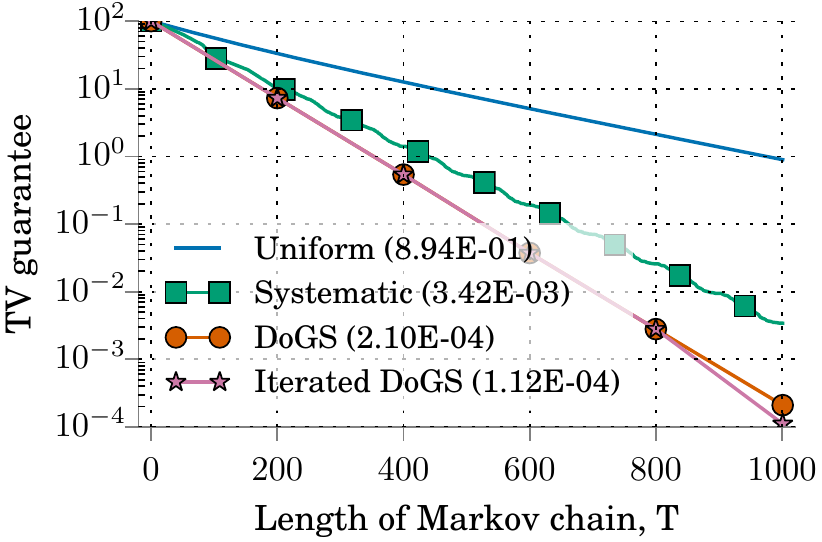}
    \includegraphics[width=0.32\textwidth]{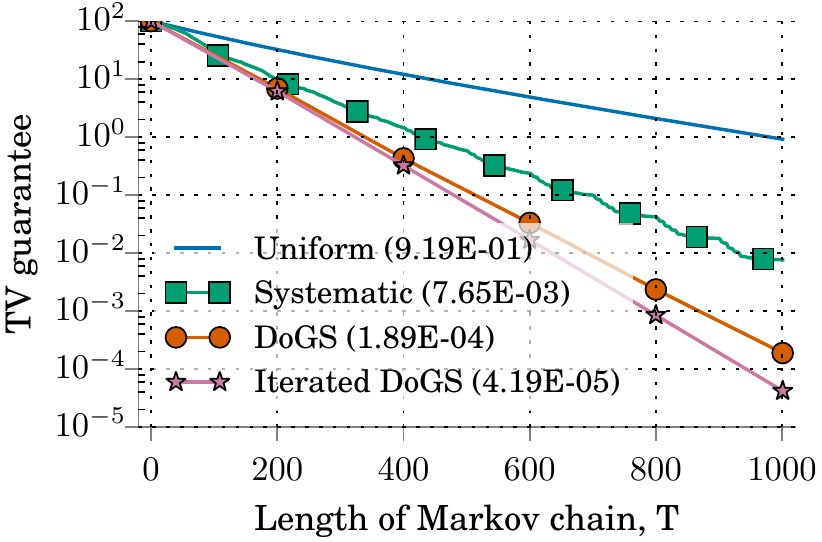}
    \includegraphics[width=0.32\textwidth]{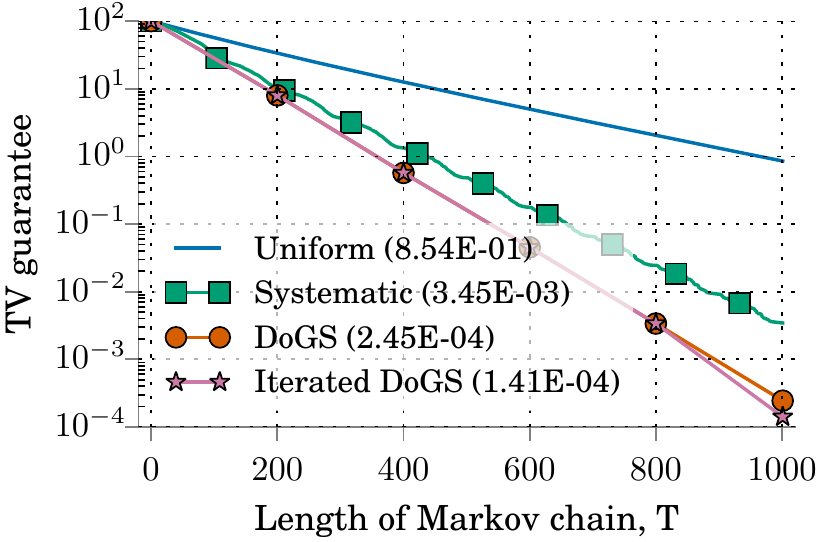}
    \includegraphics[width=0.32\textwidth]{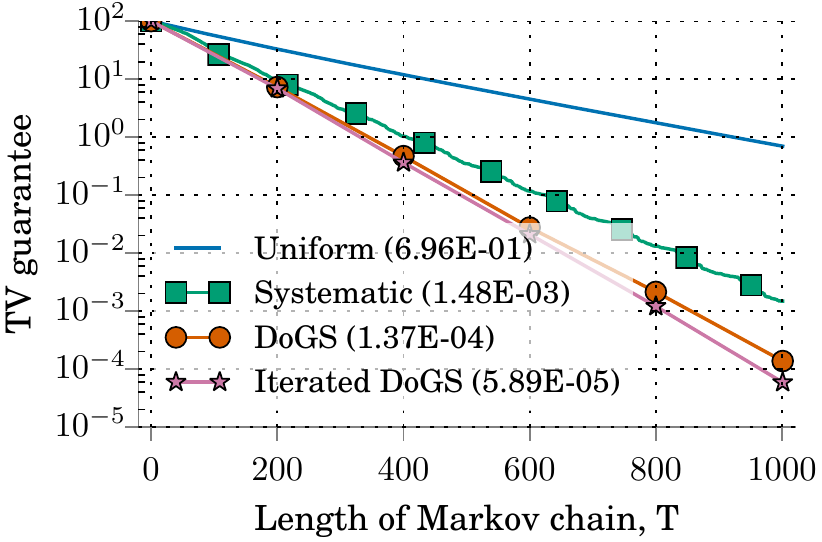}
    \includegraphics[width=0.32\textwidth]{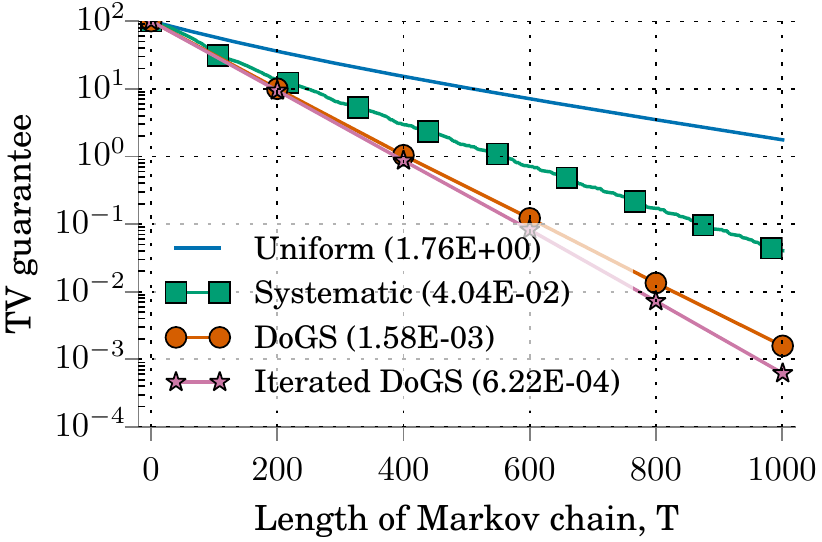}
    \includegraphics[width=0.32\textwidth]{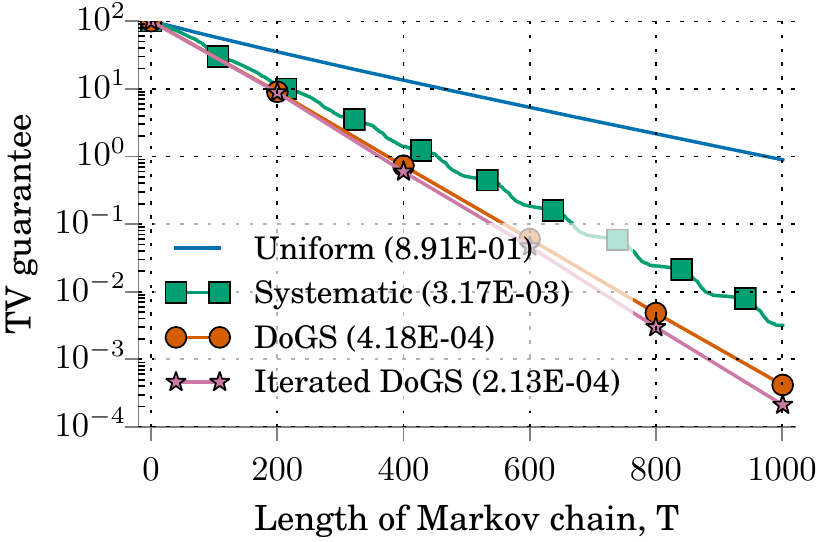}
     \includegraphics[width=0.32\textwidth]{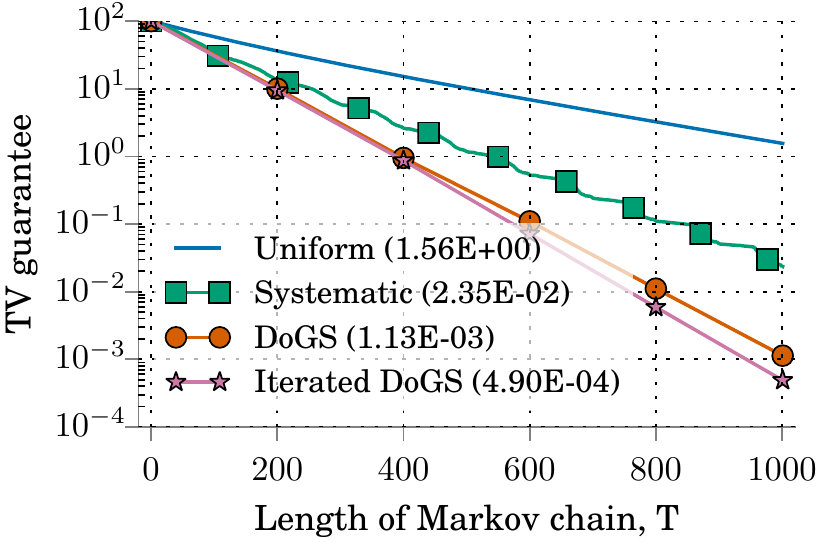}
    \includegraphics[width=0.32\textwidth]{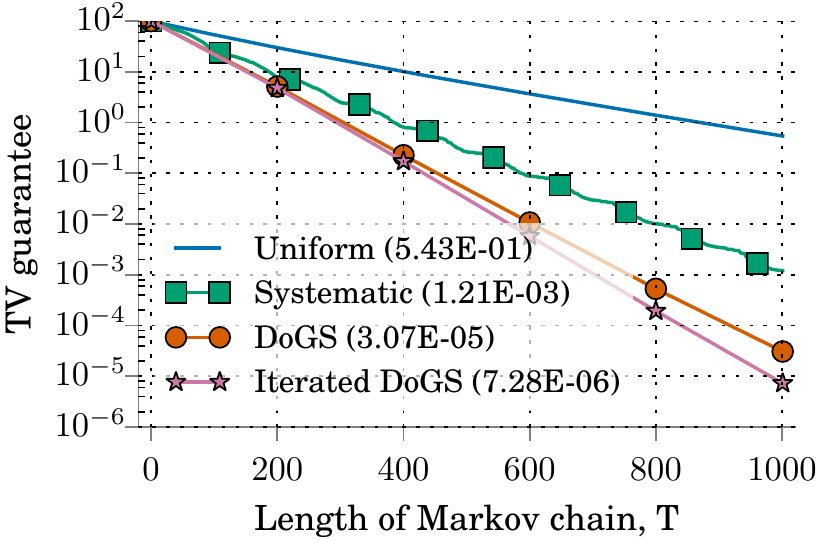}
    \includegraphics[width=0.32\textwidth]{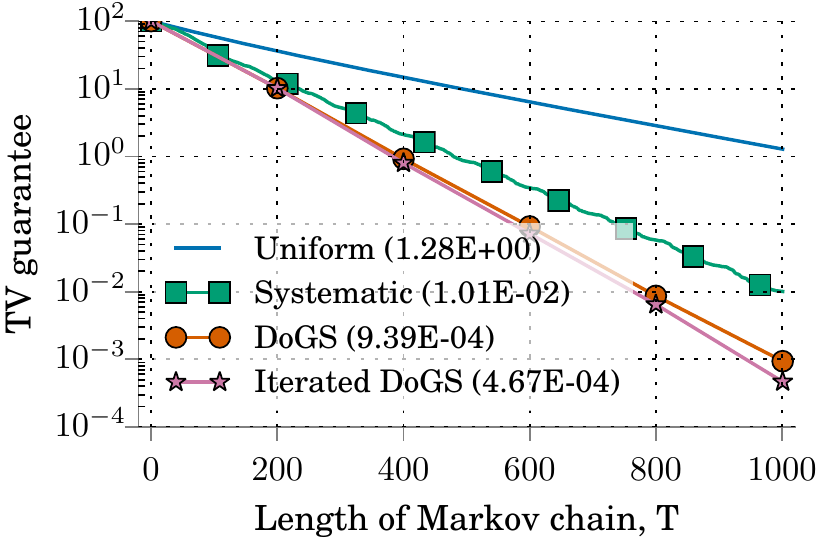}
  \end{center}
  \caption{Independent replicates of the experiment of Section~\ref{sec:eval}. }
  \label{fig:additional_ising_tv}
\end{figure}

\subsection{Independent replicates of end-to-end wall clock time performance experiment}
\label{sec:extra-wall-clock}
\label{sec:timing_appendix}

Figure~\ref{fig:isingtiming-extra} repeats the timing experiment of 
Figure~\ref{fig:isingtiming} providing three extra independent trials.
Note that when the user specifies the ``early stopping'' parameter $\eps$, \algref{coordinatedesent} terminates once its scan is within $\eps$ Dobrushin variation of the target. The long-term bias observed in the DoGS estimate of $\Esubarg{\pi}{X_1}$ is a result of this user-controlled early stopping and can be reduced arbitrarily by choosing a smaller $\eps$.
\begin{figure}[hbp]
  \begin{center}
    \includegraphics[width=0.32\textwidth]{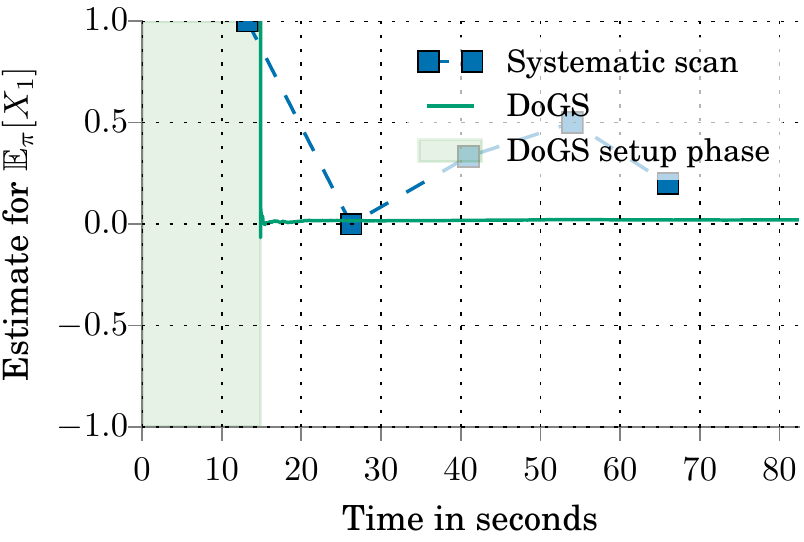}
    \includegraphics[width=0.32\textwidth]{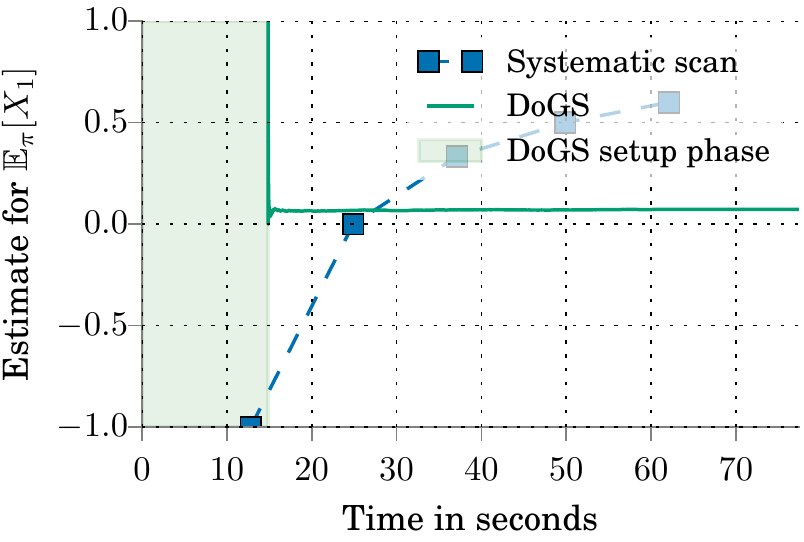}
    \includegraphics[width=0.32\textwidth]{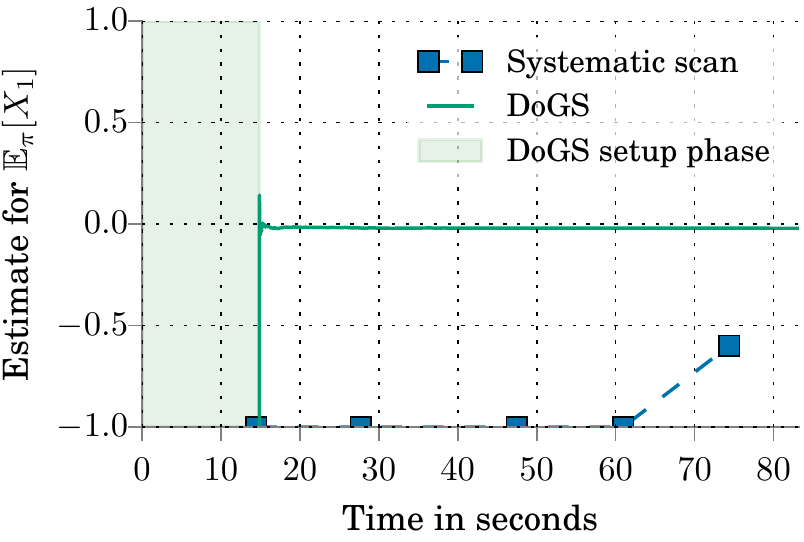}
  \end{center}
  \caption{Independent replicates of the experiment of Section~\ref{sec:wall-clock}.}
  	  \label{fig:isingtiming-extra}
\end{figure}

Figure~\ref{fig:isingtiming-small} reports the estimate of a marginal expectation versus time from a sampling experiment, including the setup time for DoGS.
The setting is the same as in the experiment of Section~\ref{sec:wall-clock} with the exception of model size:
here we simulate a $300\times300$ Ising model, with $90K$ variables in total.
The right plot uses the average measured time for a single step
the measured setup time for DoGS and the size of the two scan sequences ($190K$ for systematic, $16$ for DoGS) to give an estimate of the speedup as a function of the number of samples we draw.
\begin{figure}[hbp]
  \begin{center}
    \includegraphics[width=0.5
    \textwidth]{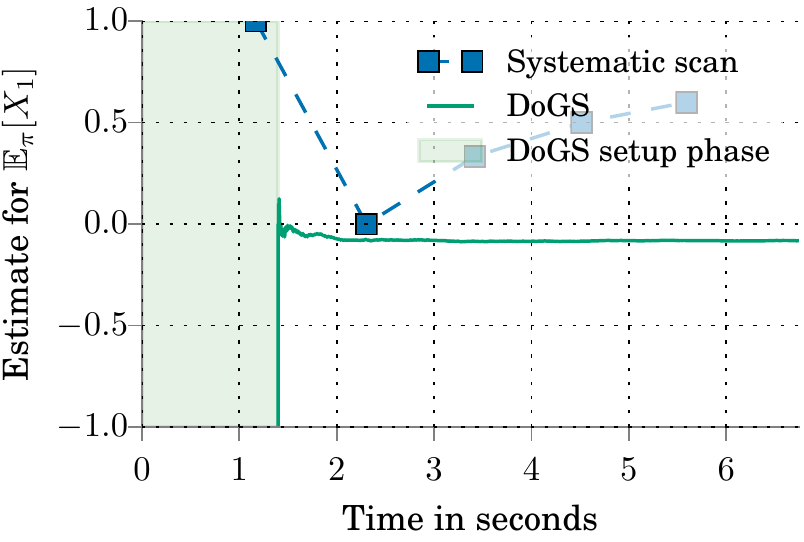}
    \includegraphics[width=0.4\textwidth]{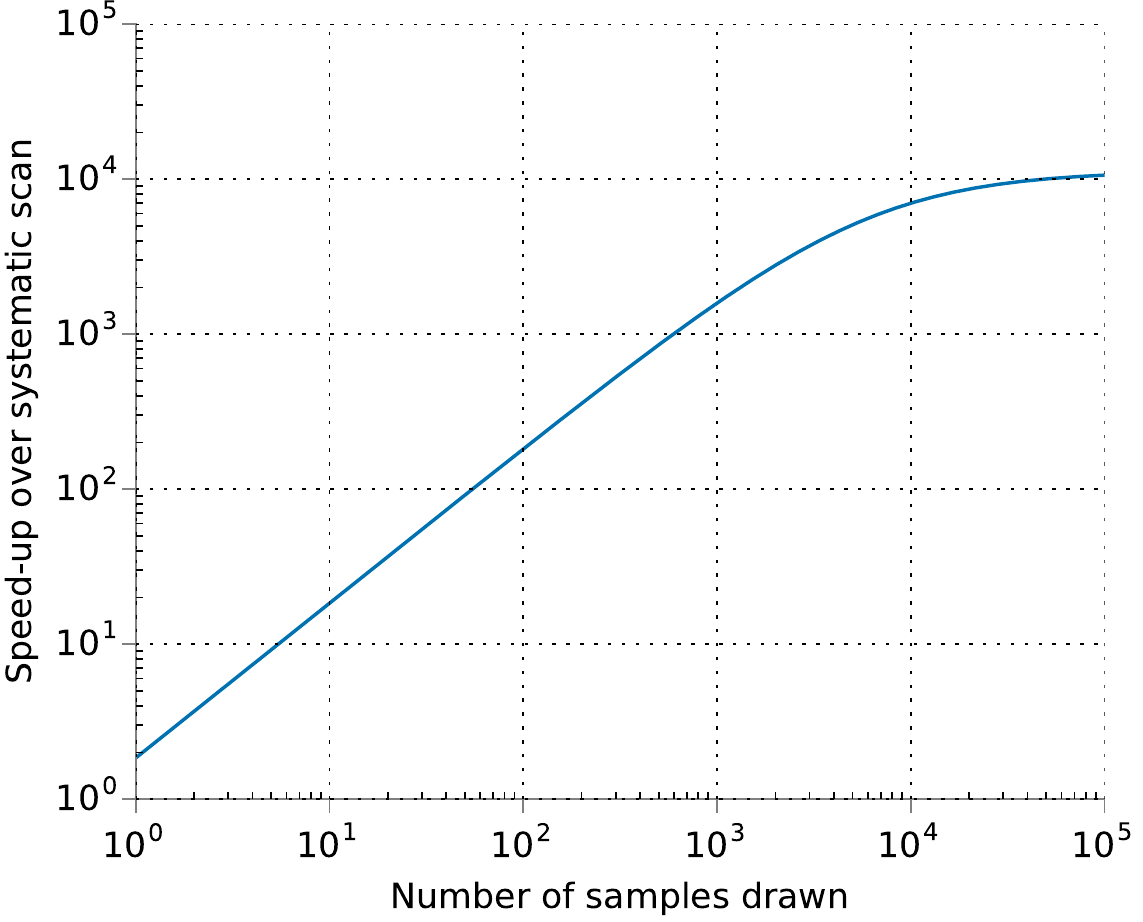}
  \end{center}
  \caption{(left) Estimate of $\mathbb{E}[x_1]$ versus wall-clock time for a standard row-major-order systematic scan and a DoGS optimized sequence on a $300\times300$ Ising model. By symmetry $\mathbb{E}[x_1]=0$. (right) The end-to-end speedup of DoGS over systematic scan, including setup and optimization time, as a function of the number of samples we draw.}
  	  \label{fig:isingtiming-small}
\end{figure}

\subsection{Addendum to customized scans for fast marginal mixing experiment}
In this section we provide alternative configurations and independent runs of the marginal experiments presented in Section~\ref{sec:expmarginal}.

Figure~\ref{fig:isingmarginalcoordtimehist} gives a spatial histogram of samples at different segments of the DoGS sequence produced in Section~\ref{sec:expmarginal}. We note that the sequence starts by sampling in the target (left) site's extended neighborhood and slowly zeroes in on the target near the end.
\newcommand{\subfigheight}{1.45in}
\begin{figure}[hbp]
  \begin{center}
    \includegraphics[height=\subfigheight]{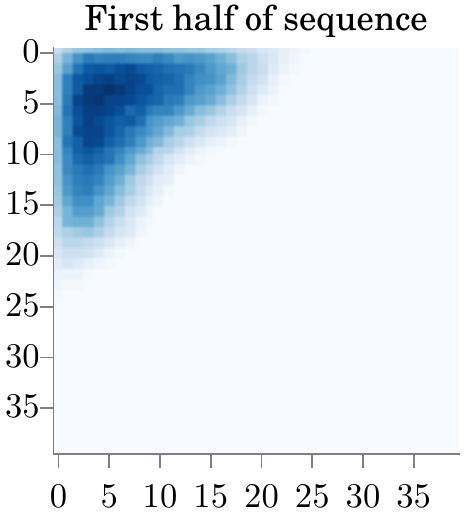}
    \includegraphics[height=\subfigheight]{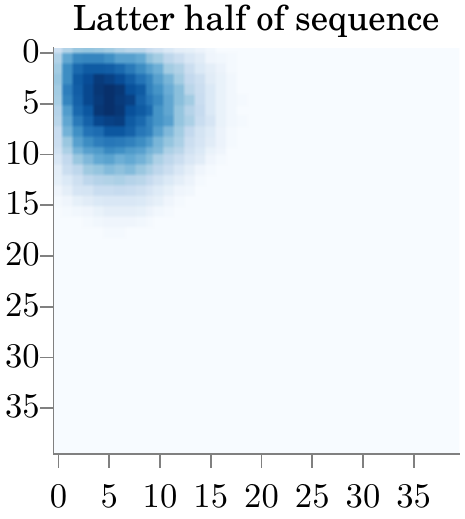}
    \includegraphics[height=\subfigheight]{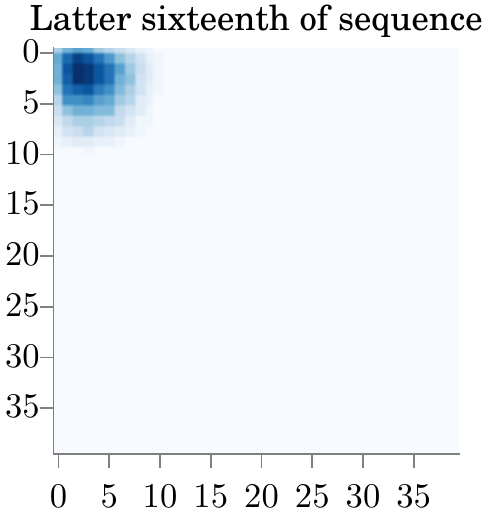}
    \includegraphics[height=\subfigheight]{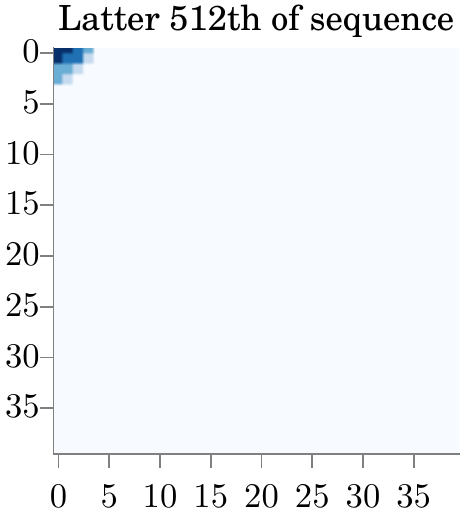}
  \end{center}
  \caption{Sampling frequencies of variables in subsequences of the scan produced by DoGS (Algorithm~\ref{alg:coordinatedesent}) in the Section~\ref{sec:expmarginal} experiment.}
  \label{fig:isingmarginalcoordtimehist}
\end{figure}

\label{sec:toroidalising}

Here we repeat the marginal Ising model experiments from Section~\ref{sec:experiments}. The set up is exactly the same, with the exception of the Ising model boundary conditions.  
Results are shown in Figure~\ref{fig:isingmarginalcoordttimehist}
and Figure~\ref{fig:isingmarginaltcoord}.

Finally, in Figure~\ref{fig:extra-marginal-sampling}, we repeat the sampling experiment of Figure~\ref{fig:samplingvsguarantees} three times.
\begin{figure}[htb]
  \begin{center}
    \includegraphics[width=0.32\textwidth]{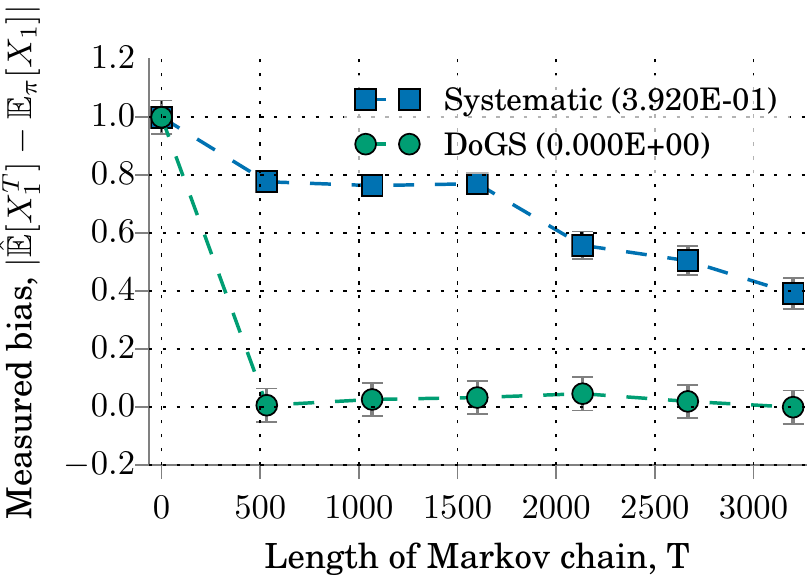}
    \includegraphics[width=0.32\textwidth]{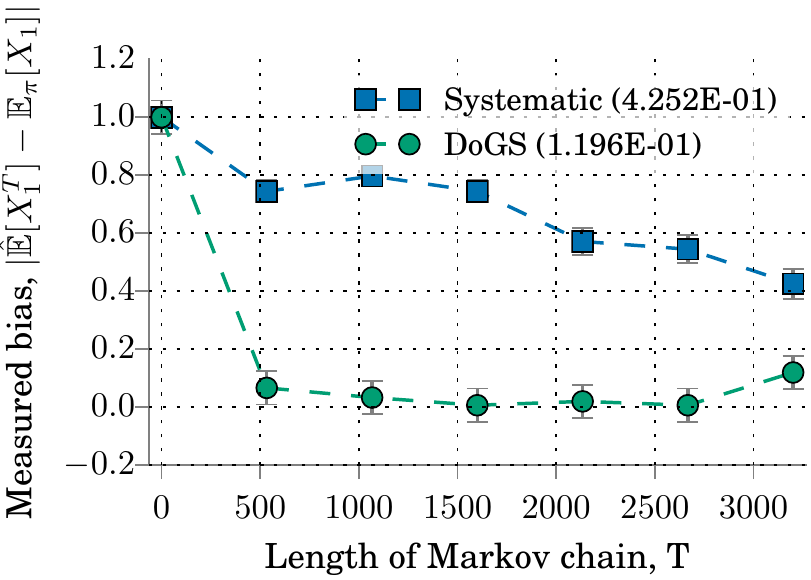}
    \includegraphics[width=0.32\textwidth]{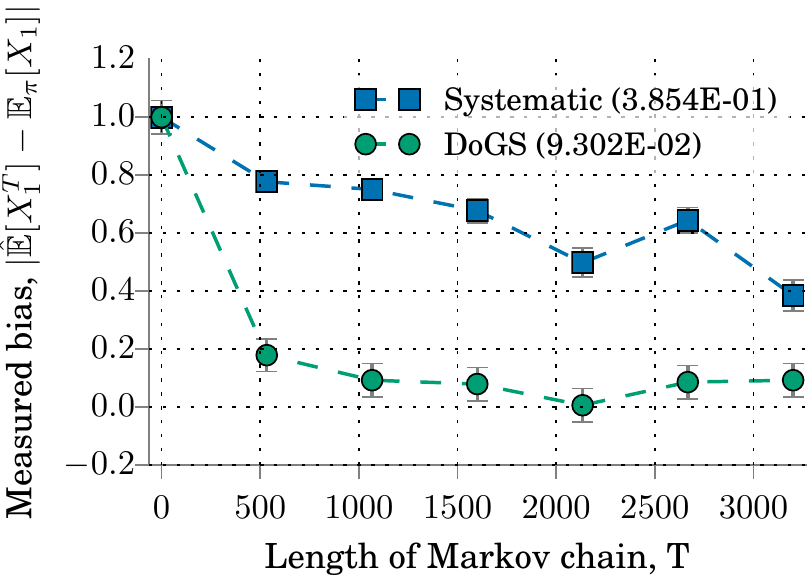}
  \end{center}
  \caption{Independent repetitions of the fast marginal mixing experiment of Section~\ref{sec:expmarginal}.
  }
\label{fig:extra-marginal-sampling}
\end{figure}


\begin{figure}[htbp]
  \begin{center}
    \includegraphics[width=0.38\textwidth]{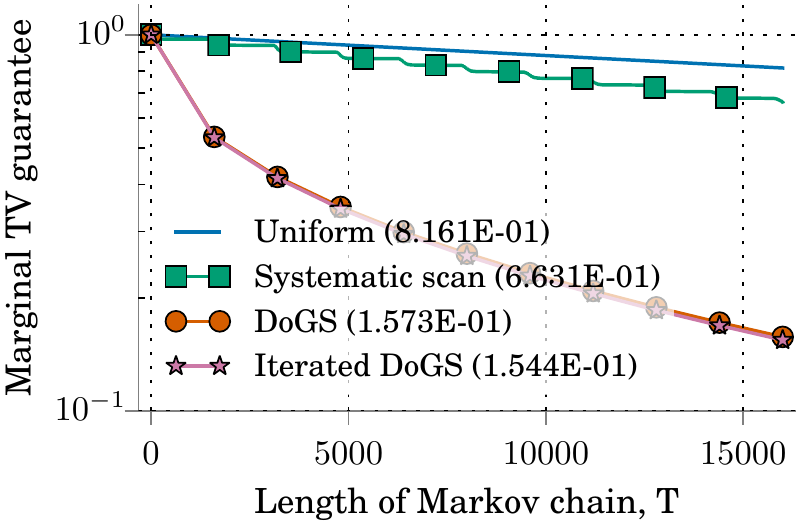}
    \includegraphics[width=0.38\textwidth]{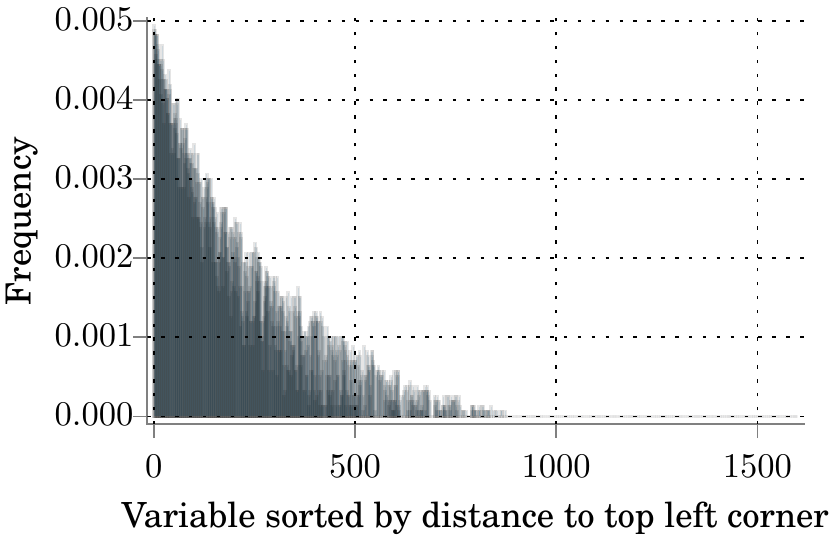}
    \includegraphics[width=0.22\textwidth]{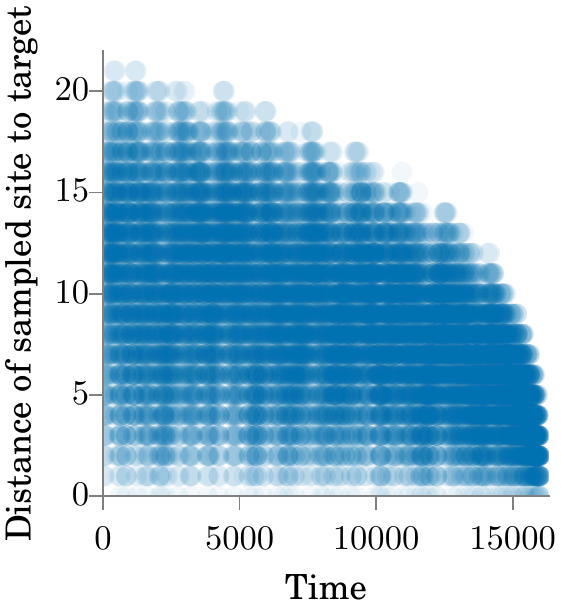}
  \end{center}
  \caption{Deterministic scan bias comparison when targetting the top left corner variable on a $40\times40$ toroidal Ising model. The middle plot shows the histogram of the sequence achieved via Algorithm~\ref{alg:coordinatedesent}. The right plot shows the sequence's distance-time profile.
}
\label{fig:isingmarginaltcoord}
\end{figure}
\begin{figure}[htbp]
  \begin{center}
    \includegraphics[height=\subfigheight]{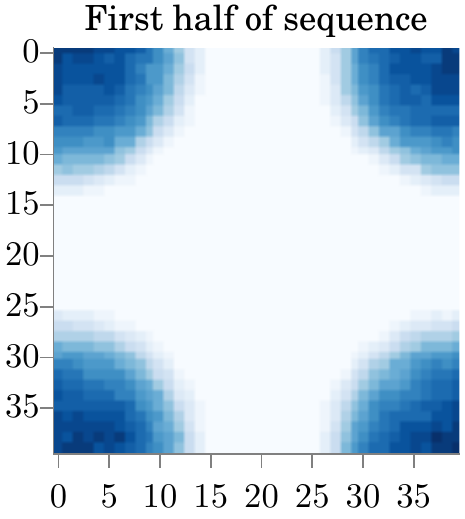}
    \includegraphics[height=\subfigheight]{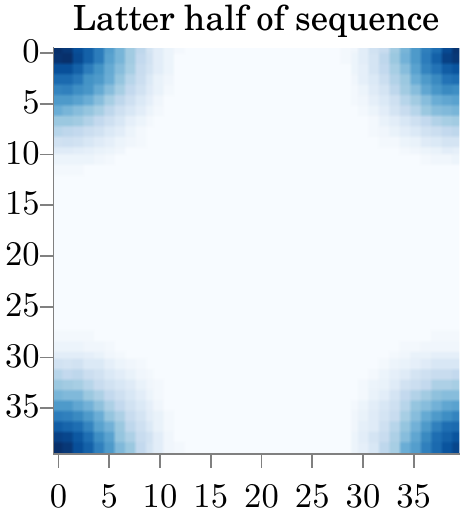}
    \includegraphics[height=\subfigheight]{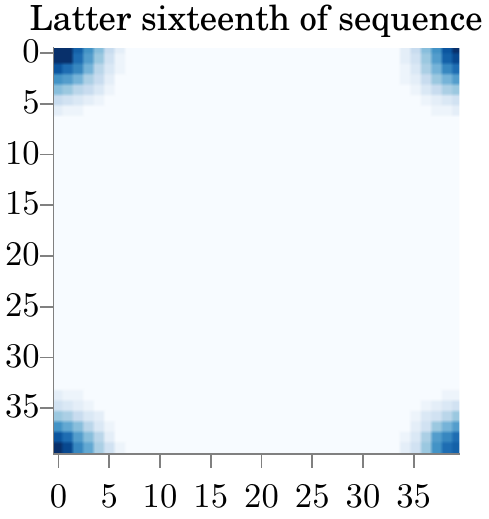}
    \includegraphics[height=\subfigheight]{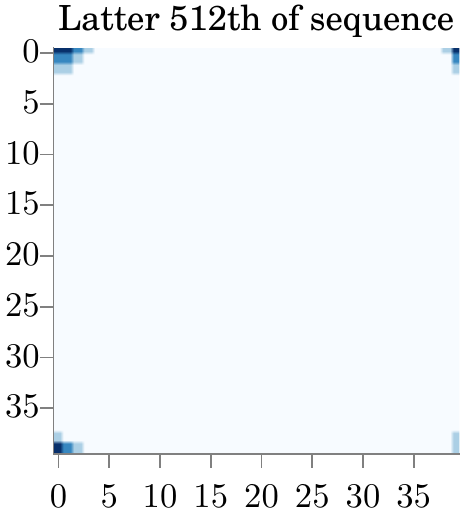}
  \end{center}
  \caption{Sampling histogram of sequence from Algorithm~\ref{alg:coordinatedesent} at different times. The left plot shows a map of the frequency at which each site on a 2D toric Ising model is sampled for the first half of a DoGS sequence, when we target the top-left variable.
  When we look at later stages of the DoGS scan sequence (later plots), DoGS samples in an ever decreasing neighborhood, zeroing in on the target site.}
  \label{fig:isingmarginalcoordttimehist}
\end{figure}


\subsection{Independent replicates of targeted image segmentation and object recognition experiment}
\label{sec:extra-segmentation}
Figure~\ref{fig:image_segmentation_hamming_extra} displays the results of two independent runs of the image segmentation experiment from 
\secref{segmentation}.
\begin{figure}[htbp]
  \begin{center}
    \includegraphics[width=0.42\textwidth]{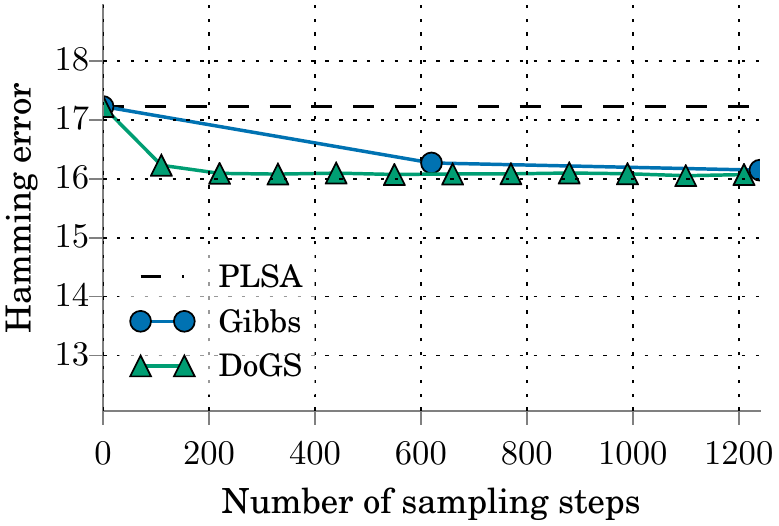}
    \includegraphics[width=0.42\textwidth]{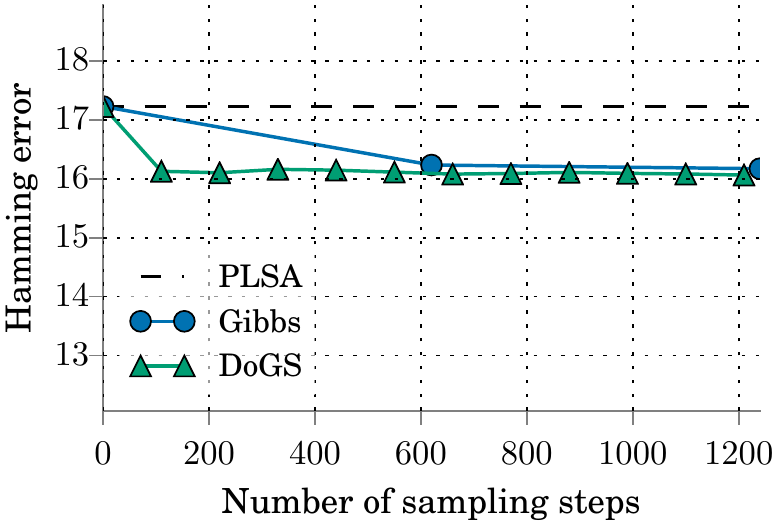}
  \end{center}
\caption{
Independent repetitions of the targeted image segmentation and object recognition experiment of \secref{segmentation}.}
  	  \label{fig:image_segmentation_hamming_extra}
\end{figure}

\end{document}